\documentclass[preprint,authoryear,12pt]{elsarticle}
\usepackage{amsfonts,amssymb,amsmath,amsopn,amsthm}
\usepackage{graphicx}
\usepackage{subcaption}
\usepackage{rotating}
\usepackage{booktabs}
\usepackage{url}
\usepackage{lscape}
\usepackage{xcolor}
\usepackage{array}
\usepackage{bbm}
\usepackage{multirow}
\usepackage{setspace}
\usepackage{tablefootnote}
\usepackage{comment}
\usepackage[space]{grffile}
\usepackage[margin=1in]{geometry}
\usepackage{xr}
\usepackage{microtype}
\usepackage{algorithm}
\usepackage{algorithmic}
\usepackage{threeparttable}
\usepackage[colorlinks=true,linkcolor=blue,citecolor=blue,urlcolor=blue]{hyperref}

\newcommand{\ignore}[1]{}

\theoremstyle{plain}
\newtheorem{theorem}{Theorem}[section]
\newtheorem{proposition}[theorem]{Proposition}
\newtheorem{lemma}[theorem]{Lemma}
\newtheorem{corollary}[theorem]{Corollary}

\theoremstyle{definition}
\newtheorem{definition}[theorem]{Definition}
\newtheorem{assumption}[theorem]{Assumption}

\theoremstyle{remark}

\newcommand{\E}{\mathbb{E}}
\newcommand{\Pp}{\mathbb{P}}
\newcommand{\R}{\mathbb{R}}

\newcommand{\1}{\mathbf{1}}
\newcommand{\G}{\mathcal{G}}
\newcommand{\M}{\mathcal{M}}
\newcommand{\Vset}{\mathcal{V}}
\newcommand{\Cset}{\mathcal{C}}

\newcommand{\Iset}{\mathcal{I}}
\newcommand{\argmin}{\operatorname*{arg\,min}}
\newcommand{\pos}{\mathrm{pos}}
\newcommand{\cert}{\mathrm{cert}}
\newcommand{\proj}{\mathrm{proj}}
\newcommand{\CCSMCS}{\textnormal{\textsc{CC-SMCS}}}

\newcommand{\bigtimes}{\mathop{\times}}

\newcommand{\qcert}{q^{\mathrm{cert}}}

\newcommand{\papertitle}{Sequential Confidence Sets for Coverage-Constrained Conformal Model Selection}

\begin{document}
\begin{frontmatter}

\title{\papertitle}
\author{Jing Li\footnote{Department of Statistics and Data Science, School of Economics, Jinan University, Guangzhou, China. Email: \href{mailto:ssun28769@stu2025.jnu.edu.cn}{ssun28769@stu2025.jnu.edu.cn} }}
\author{Haibin Zhu\footnote{Corresponding author. Department of Statistics and Data Science, School of Economics, Jinan University, Guangzhou, China. Email: \href{mailto:haibinzhu@jnu.edu.cn}{haibinzhu@jnu.edu.cn} }}

\begin{abstract}
Modern conformal forecasting systems often maintain several adaptive pipelines that differ in base forecasters, conformity scores, calibration windows, and update rules. Comparing them is difficult because coverage is a hard constraint, whereas efficiency should be optimized only among feasible pipelines. We formulate this problem as sequential inference for a stochastic constrained argmin. At each time step, the target is the set of minimum-cost pipelines satisfying multiple prefix-average conditional miscoverage constraints. We introduce Coverage-Constrained Sequential Model Confidence Sets (CC-SMCS), which separate certifiably feasible, possibly feasible, and possibly constrained-optimal pipelines. Using simultaneous martingale confidence sequences, CC-SMCS projects a rectangular confidence region onto the constrained argmin and admits an exact closed-form rule. With probability at least $1-\delta$, it contains every constrained-optimal pipeline simultaneously over all times. This finite-sample guarantee requires no stationarity or mixing assumptions and remains valid under data-dependent stopping. We also establish an impossibility result for safe certification at the coverage boundary and extend the construction to delayed multi-horizon feedback and outcome-dependent efficiency objectives.
\end{abstract}
\end{frontmatter}

\doublespacing
\section{Introduction}
\label{sec:introduction}
Conformal prediction (CP) turns point or distributional predictors into set-valued predictions with explicit coverage targets \citep{vovk2005algorithmic,angelopoulos2023gentle}. In practice, however, a forecasting system rarely relies on a single conformal pipeline. Different base forecasters, conformity scores, calibration windows, nominal levels, and online update rules can all produce prediction sets with different coverage and efficiency, and their relative performance may change under distribution shift. The resulting problem is how to compare a library of adaptive conformal procedures while accounting for uncertainty about both coverage and efficiency.

The main difficulty is that coverage and efficiency play different roles. Coverage is a constraint: a narrow prediction set is not useful if it systematically undercovers. Efficiency, by contrast, should be optimized only among pipelines that satisfy the required coverage constraints. A scalar score that combines width and undercoverage replaces this constrained problem by a different objective and can favor an under-covering method. A two-stage rule that first screens methods for coverage and then selects the narrowest survivor has a different problem: failure to reject infeasibility is not evidence of feasibility. Near the coverage boundary, an infeasible but narrow method may survive because there is not yet enough evidence to exclude it, while a feasible boundary method may be difficult to certify.

This motivates a different inferential target. Rather than selecting a single pipeline, we ask which pipelines can still be the optimizer of a coverage-constrained problem at each time. At time $t$, every candidate pipeline has several prefix-average conditional miscoverage constraints and an observed prequential efficiency cost. A pipeline is feasible when all monitored coverage constraints are satisfied, and the target $\M_t^\star$ is the set of feasible pipelines with minimum cost. Because both the feasible set and the identity of the optimizer may change over time, we seek a sequence of random model sets $(\widehat\M_t^{\proj})_{t\geq1}$ satisfying
\begin{equation}
\Pp\!\left(\forall t\geq1:\ \M_t^\star\subseteq\widehat\M_t^{\proj}\right)\geq1-\delta.
\label{eq:introgoal}
\end{equation}
We introduce a \emph{Coverage-Constrained Sequential Model Confidence Set} (\CCSMCS) for this target. It separates certified and possible feasibility and projects simultaneous coverage confidence sequences onto the constrained argmin. The rectangular projection has a closed-form rule computable in $O(MR)$ time. Under prequential measurability and bounded miscoverage increments, (1) holds without stationarity, mixing, or exchangeability assumptions and remains valid under data-dependent stopping.

The target is method comparison: membership in $\widehat{\mathcal M}^{\mathrm{proj}}_t$ does not imply certified feasibility or predictive coverage after selection.

\paragraph{Contributions.}
\begin{enumerate}
    \item We formulate conformal pipeline comparison as sequential inference for a stochastic constrained argmin with an unknown, time-varying feasible set. The framework supports multiple coverage constraints, with extensions to delayed multi-horizon feedback and outcome-dependent efficiency objectives.
    \item We introduce \CCSMCS, which separates certifiably feasible, possibly feasible, and possibly constrained-optimal pipelines. For an observed prequential objective, we derive an exact rectangular projection rule and prove finite-sample, time-uniform oracle coverage using martingale confidence sequences.
    \item We give sufficient conditions for oracle recovery and a finite-sample lower bound near the coverage boundary. At the boundary, we show that even a known strict cost advantage need not permit high-probability singleton identification under uniform oracle coverage.
\end{enumerate}

\section{Related Work}
\label{sec:related}

\paragraph{Model confidence sets and argmin inference.} The classical Model Confidence Set (MCS) contains the best models with asymptotic confidence for a fixed evaluation sample \citep{hansen2011mcs}. Sequential MCS provides nonasymptotic time-uniform coverage for models ranked by scalar conditional-average losses \citep{arnold2026sequential}. Related methods test sequential predictive ability using betting processes \citep{choe2024comparing}. Model Prediction Sets target the empirically best model for the next period under a long-run model-set coverage criterion \citep{li2025mps}. Other work studies confidence sets for the index of the minimum component of a noisy mean vector \citep{zhang2024winners,kim2025bestmodel}. We instead minimize an observed objective over an unknown, time-varying feasible set.

\paragraph{Online conformal prediction and sequential risk control.} Conformal methods address temporal dependence and distribution shift \citep{barber2023conformal,xu2021enbpi,zaffran2022adaptive}. Adaptive conformal inference controls long-run coverage error \citep{gibbs2021adaptive,gibbs2024arbitrary}, while strongly adaptive methods consider multiple time scales \citep{bhatnagar2023improved}. Risk-control extensions address non-exchangeable data \citep{farinhas2024non} and arbitrary inspection times \citep{hultberg2026anytime}. Multi-model methods learn a predictor to deploy \citep{hajihashemi2024multi}, combine conformity scores through Bayesian model averaging \citep{bhagwat2025cbma}, or weight residuals using gating similarities \citep{kong2026adaptive}. These methods construct or deploy prediction sets. \CCSMCS\ instead compares pipelines under monitored coverage constraints.

\paragraph{Risk-controlled selection and constrained identification.} Learn-then-test treats risk control as multiple hypothesis testing \citep{angelopoulos2025learn}, and Pareto Testing combines risk certification with optimization \citep{laufer2023pareto}. Conformal risk control calibrates prediction sets for general monotone losses \citep{angelopoulos2024conformal}. Constrained best-arm identification uses adaptive sampling to identify a fixed optimizer under unknown feasibility constraints \citep{lardy2025constrained,cai2026constrained}. We instead evaluate adaptive pipelines on a common outcome stream and cover all current constrained optimizers over time, allowing excluded pipelines to re-enter.

\paragraph{Conformal model selection and coverage assessment.} Efficiency-oriented selection motivates procedures that preserve predictive validity after selection \citep{liang2024selection,hegazy2025valid,wang2026localized}. Conformal Prediction Assessment evaluates conditional coverage and develops a selection criterion using a learned reliability estimator \citep{zhou2026assessment}. Our target is sequential method comparison under prefix-average conditional coverage constraints.

\section{Setup and Target}
\label{sec:setup}


Let $(\Omega,\mathcal A,\Pp)$ be a probability space with an increasing filtration $(\G_t)_{t\geq0}$. The sigma-field $\G_{t-1}$ contains all information available when the prediction for outcome $Y_t$ is issued.

There are $M<\infty$ candidate conformal pipelines indexed by $[M]$ and $R<\infty$ monitored constraints indexed by $[R]$. A candidate is the complete conformal procedure, including its base forecaster, conformity score, calibration sample or rolling window, nominal level, update rule, hyperparameters, and randomization. Thus, candidates may share the same base forecaster while differing in their conformal layers. Each pipeline may adapt arbitrarily to the past, but its internal state and output at time $t$ must be $\G_{t-1}$-measurable. 



For every method $m\in[M]$ and constraint $r\in[R]$, let $C_{mrt}$ be the prediction set for $Y_t$. A conformal pipeline is designed to produce prediction sets at a prescribed nominal coverage level. Our comparison framework evaluates these sets through the monitored coverage constraints defined below. For notational simplicity, the index $r$ may encode a nominal level, a tail, a forecast horizon after the delayed-feedback construction, or a predictable subgroup.

Define the miscoverage indicator as
\begin{equation}
B_{mrt}=\1\{Y_t\notin C_{mrt}\}.
\label{eq:miss}
\end{equation}
Let $\tau_r\in(0,1)$ be the tolerated miscoverage rate, and let $a_{rt}\in[0,1]$ be a $\G_{t-1}$-measurable exposure weight. For example, $a_{rt}=1$ gives an overall constraint, while $a_{rt}=\1\{S_t=s\}$ selects a regime $S_t$ known at prediction time. Define 
\begin{align*}
X_{mrt}&=a_{rt}(B_{mrt}-\tau_r),
&\mu_{mrt}&=\E[X_{mrt}\mid\G_{t-1}],\\
A_{r,t}&=\sum_{s=1}^t a_{rs},
&\Gamma_{mr,t}&=
\begin{cases}
A_{r,t}^{-1}\sum_{s=1}^t\mu_{mrs},&A_{r,t}>0,\\
0,&A_{r,t}=0.
\end{cases}
\end{align*}
The inequality $\Gamma_{mr,t}\leq 0$ requires the exposure-weighted prefix average of the conditional miscoverage probabilities to be at most $\tau_r$. For a subgroup indicator, it is the average of the conditional risks over the periods when that subgroup is active. The target differs from usual marginal conformal coverage and is weaker than pointwise conditional coverage.

In practice, we allow for a prespecified tolerance
\begin{equation}\label{eq:tolerance}
\tau_r=\alpha_r+\varepsilon_r,\qquad \varepsilon_r\geq0,
\end{equation}
where $1-\alpha_r$ is the nominal conformal level. The quantity $\varepsilon_r$ is a prespecified practical tolerance that relaxes the coverage requirement. Section~\ref{sec:coverage-boundary} explains its role in certification.

\subsection{Observed prequential efficiency}

Let $c_{mt}\geq0$ be a $\G_{t-1}$-measurable set cost, such as interval length, Lebesgue measure, classification-set cardinality, or a prespecified weighted average of these quantities across levels and horizons. Let $w_t\geq0$ be predictable weights and assume $W_t=\sum_{s=1}^t w_s>0$. Define the weighted prequential cost as
\begin{equation}
Q_{m,t}=\frac{1}{W_t}\sum_{s=1}^t w_s c_{ms}.
\label{eq:objective}
\end{equation}
The objective is observed exactly because the set cost is known when the prediction is issued. This design isolates uncertainty about feasibility. Appendix~\ref{app:randomobjective} extends the method to outcome-dependent efficiency losses by adding confidence sequences for pairwise objective differences.

\begin{definition}[Dynamic feasible and oracle sets]
\label{def:oracle}
At time $t$, define
\[
\Vset_t^\star=\{m\in[M]:\Gamma_{mr,t}\leq0\ \text{for every }r\in[R]\},\qquad \M_t^\star=\argmin_{m\in\Vset_t^\star}Q_{m,t},
\]
with $\M_t^\star=\varnothing$ if $\Vset_t^\star=\varnothing$.
\end{definition}

Both sets are random and may change with $t$. Multiple pipelines may tie for the constrained optimum. Importantly, the target does not assume that any candidate is a true data-generating model or that the process is stationary.

\begin{definition}[Coverage-constrained sequential model confidence set]
A sequence $(\widehat\M_t^{\proj})_{t\geq1}$ is a level-$1-\delta$ \CCSMCS\ for $(\M_t^\star)_{t\geq1}$ if
\begin{equation}
\Pp\!\left(\forall t\geq1:\ \M_t^\star\subseteq\widehat\M_t^{\proj}\right)\geq1-\delta.
\label{eq:ccsmcsdef}
\end{equation}
\end{definition}

We call \eqref{eq:ccsmcsdef} \emph{finite-sample time-uniform coverage}. Its operational consequence is anytime validity: the analyst may inspect the set repeatedly and stop according to the observed data without invalidating the guarantee.

\section{Coverage-Constrained Sequential Model Confidence Sets}
\label{sec:method}

\subsection{Simultaneous confidence sequences for coverage risk}

Let
\begin{equation}
S_{mr,t}=\sum_{s=1}^tX_{mrs}.
\label{eq:SV}
\end{equation}
Following the empirical-Bernstein construction of \citet{howard2021time}, define the predictable estimate
\begin{equation}
\widehat B_{mrt}=
\begin{cases}
\tau_r,&A_{r,t-1}=0,\\
\dfrac{\sum_{s=1}^{t-1}a_{rs}B_{mrs}}{A_{r,t-1}},&A_{r,t-1}>0,
\end{cases}
\qquad
\widehat X_{mrt}=a_{rt}(\widehat B_{mrt}-\tau_r).
\label{eq:predictable_estimate}
\end{equation}
By construction, $\widehat B_{mrt}$ is $\G_{t-1}$-measurable and belongs to $[0,1]$. Since $a_{rt}\in[0,1]$, both $X_{mrt}$ and $\widehat X_{mrt}$ belong to the fixed interval $[-\tau_r,1-\tau_r]$, whose length is one. Define the empirical intrinsic-time process
\begin{equation}
V_{mr,t}=\sum_{s=1}^t(X_{mrs}-\widehat X_{mrs})^2.
\label{eq:empirical_variance}
\end{equation}
For $\rho>0$ and a two-sided error probability $\eta\in(0,1)$, let $b_{\eta,\rho}$ denote the gamma-exponential mixture boundary defined in Appendix~\ref{app:mixtureproof}. The empirical-Bernstein construction gives
\begin{equation}
\Pp\left(\forall t\geq1:\left|S_{mr,t}-\sum_{s=1}^t\mu_{mrs}\right|\leq b_{\eta,\rho}(V_{mr,t})\right)\geq1-\eta.
\label{eq:singlecs}
\end{equation}
Lemma~\ref{lem:mixed_bound} in Appendix~\ref{app:mixtureproof} proves \eqref{eq:singlecs}. Choose fixed weights $\pi_{mr}>0$ satisfying $\sum_{m,r}\pi_{mr}=1$, and set $\eta_{mr}=\delta\pi_{mr}$. For $A_{r,t}>0$, define
\begin{equation}
L_{mr,t}=\frac{S_{mr,t}-b_{\eta_{mr},\rho}(V_{mr,t})}{A_{r,t}},\qquad U_{mr,t}=\frac{S_{mr,t}+b_{\eta_{mr},\rho}(V_{mr,t})}{A_{r,t}}.
\label{eq:LU}
\end{equation}
If $A_{r,t}=0$, set $L_{mr,t}=-\infty$ and $U_{mr,t}=+\infty$. A union bound over the method--constraint pairs yields
\begin{equation}
\Pp\left(\forall t,m,r:\ L_{mr,t}\leq\Gamma_{mr,t}\leq U_{mr,t}\right)\geq1-\delta.
\label{eq:jointcs}
\end{equation}
Lemma~\ref{lem:gamma_interval} in Appendix~\ref{app:mixtureproof} proves \eqref{eq:jointcs}.

\subsection{Certified and possibly feasible methods}

A method $m$ is ruled infeasible if there exists at least one constraint $r$ such that $L_{mr,t}>0$. On the joint confidence event, this implies $\Gamma_{mr,t}>0$, so method $m$ violates at least one coverage constraint. The set of methods that have not been ruled infeasible is
\begin{equation}
\widehat\Vset_t^{\pos}=\{m:L_{mr,t}\leq0\ \text{for every }r\}.
\label{eq:possible}
\end{equation}
Among these methods, the certified feasible set is
\begin{equation}
\widehat\Vset_t^{\cert}=\{m:U_{mr,t}\leq0\ \text{for every }r\}.
\label{eq:certified}
\end{equation}
A method in $\widehat\Vset_t^{\pos}$ is possibly feasible because every confidence interval contains at least one nonpositive value, while a method in $\widehat\Vset_t^{\cert}$ is certified feasible because every value in each confidence interval is nonpositive. Therefore, $\widehat\Vset_t^{\cert}\subseteq\widehat\Vset_t^{\pos}$. On the joint confidence event, $\widehat\Vset_t^{\cert}\subseteq\Vset_t^\star\subseteq\widehat\Vset_t^{\pos}$ holds for every $t$.
If a constraint has received no exposure, its confidence interval is $(-\infty,+\infty)$. It does not rule out feasibility, since its lower bound is nonpositive, but it prevents certification since its upper bound is positive.

\subsection{Projection onto the constrained argmin}
\label{meth:simple_way}

At time $t$, define the rectangular confidence region $\Cset_t=\bigtimes_{m=1}^{M}\bigtimes_{r=1}^{R}[L_{mr,t},U_{mr,t}]$.
For $\gamma=(\gamma_{mr})\in\Cset_t$, let
\[
\Vset(\gamma)=\{m\in[M]:\gamma_{mr}\leq0\ \text{for every }r\in[R]\},
\qquad
\M^\star(\gamma)=\argmin_{m\in\Vset(\gamma)}Q_{m,t}
\]
when $\Vset(\gamma)\neq\varnothing$. Define the projection confidence set by
\begin{equation}
\widehat\M_t^{\proj}=\bigcup_{\substack{\gamma\in\Cset_t\\\Vset(\gamma)\neq\varnothing}}\M^\star(\gamma).
\label{eq:projection}
\end{equation}
The rectangular structure yields the following closed-form characterization.

\begin{proposition}\label{prop:projection}
Fix any time $t$. Suppose that the coordinate intervals in $\Cset_t$ are nonempty and each contains at least one finite value, and that $Q_{m,t}\in\R$ is fixed for every $m$. Define
\begin{equation}
\qcert_t=\min_{j\in\widehat\Vset_t^{\cert}}Q_{j,t},\qquad\min\varnothing=+\infty.
\label{eq:qcert}
\end{equation}
Then
\begin{equation}
\widehat\M_t^{\proj}=\left\{m\in\widehat\Vset_t^{\pos}:~Q_{m,t}\leq\qcert_t\right\}.
\label{eq:operational}
\end{equation}
\end{proposition}

The result is deterministic and uses only the rectangular confidence region and the fixed objective values. Appendix~\ref{app:projectionproof} gives the proof. We therefore define the operational \CCSMCS\ by the right-hand side of \eqref{eq:operational}, and Algorithm~\ref{alg:ccsmcs} summarizes the full procedure.

\section{Theoretical Guarantees}
\label{sec:theory}

The following assumption collects the measurability and design conditions used for the confidence sequences and the main coverage result.
\begin{assumption}\label{ass:protocol}
The integers $M,R$, tolerances $\tau_r$, confidence level $\delta\in(0,1)$, allocations $\pi_{mr}>0$ with $\sum_{m,r}\pi_{mr}=1$, objective definition, and mixture parameter $\rho>0$ are fixed before evaluation. For every $t,m,r$, the set $C_{mrt}$, exposure $a_{rt}$, cost $c_{mt}$, and weight $w_t$ are $\G_{t-1}$-measurable; $Y_t$ and the resulting miss indicators are $\G_t$-measurable. The costs are finite and $\sum_{s=1}^t w_s>0$ whenever $Q_{m,t}$ is reported.
\end{assumption}

\subsection{Time-uniform coverage and optional stopping}

\begin{theorem}\label{thm:main}
Under Assumption~\ref{ass:protocol}, construct the confidence sequences in \eqref{eq:LU} with fixed allocations $\eta_{mr}=\delta\pi_{mr}$. Then
\begin{equation}
\Pp\!\left(\forall t\geq1:\ \widehat\Vset_t^{\cert}\subseteq\Vset_t^\star \subseteq\widehat\Vset_t^{\pos}, \quad \M_t^\star\subseteq\widehat\M_t^{\proj}\right)\geq1-\delta.
\label{eq:mainresult}
\end{equation}
\end{theorem}

Because Theorem~\ref{thm:main} holds simultaneously over all times, no additional correction is needed for repeated inspection or data-dependent stopping. For any stopping time $T$, the probability that $T<\infty$ and $\M_T^\star\nsubseteq\widehat\M_T^{\proj}$ is at most $\delta$. Examples include stopping when $|\widehat\M_t^{\proj}|=1$, when a prespecified method is certified, or when a computational budget is exhausted.
\subsection{Contraction of the oracle confidence set}
\label{sec:oracle-contraction}

For $A_{r,t}>0$, define
$\mathrm{rad}_{mr,t}:=b_{\eta_{mr},\rho}(V_{mr,t})/A_{r,t}$,
and set $\mathrm{rad}_{mr,t}=+\infty$ when $A_{r,t}=0$.
Under persistent exposure $A_{r,t}\asymp t$ almost surely for every $r$
and balanced allocations $\eta_{mr}\asymp\delta/(MR)$,
Proposition~B.5 gives
\begin{equation}
\mathrm{rad}_{mr,t}=O\!\left(\sqrt{\frac{\log t+\log(MR/\delta)}{t}}\right)\longrightarrow 0.
\label{eq:oracle-contraction-rate}
\end{equation}
On the joint confidence event, method $m$ is excluded at time $t$
whenever either
\[
\begin{aligned}
&\exists r:\quad
\Gamma_{mr,t}>2\mathrm{rad}_{mr,t},\qquad\text{or}\\
&\exists j:\quad
Q_{j,t}<Q_{m,t}\quad\text{and}\quad\Gamma_{jr,t}+2\mathrm{rad}_{jr,t}\le0\quad\text{for every }r.
\end{aligned}
\]
Thus $\widehat{\mathcal M}^{\mathrm{proj}}_t=\mathcal M^\star_t$
whenever every non-oracle method satisfies either condition.
Under persistent exposure, this equality holds eventually if there is
$\zeta>0$ such that, for all sufficiently large $t$, each non-oracle
method either violates a constraint by at least $\zeta$ or has a
strictly cheaper competitor satisfying
$\Gamma_{jr,t}\le-\zeta$ for every $r$.
Appendix~\ref{app:contraction} proves these statements.

\ignore{
\subsection{Contraction away from the coverage boundary}\label{sec:contraction}

For $A_{r,t}>0$, define $\operatorname{rad}_{mr,t}:=\frac{b_{\eta_{mr},\rho}(V_{mr,t})}{A_{r,t}}$. Under persistent exposure, with $A_{r,t}\asymp t$ almost surely, Proposition~\ref{prop:contraction} in Appendix~\ref{app:contraction} gives
\begin{equation}
\operatorname{rad}_{mr,t}=O\left(\sqrt{\frac{\log t+\log(MR/\delta)}{t}}\right)\to0.
\label{eq:rad}
\end{equation}
for balanced allocations $\eta_{mr}\asymp\delta/(MR)$. On the joint confidence event, if $\Gamma_{mr,t}\leq-\zeta$ for all sufficiently large $t$, the constraint is eventually certified. If $\Gamma_{mr,t}\geq\zeta$ for all sufficiently large $t$, it eventually rules the method infeasible. The remaining difficulty is the boundary case $\Gamma_{mr,t}=0$.}

\subsection{The coverage boundary}\label{sec:coverage-boundary}

Section~\ref{sec:oracle-contraction} shows how feasibility margins
allow CC-SMCS to remove non-oracle methods.
At the boundary, even a known strict cost advantage need not suffice.
Consider $Y_t=B_t$ with $(B_t)_{t\ge1}$ i.i.d.\
$\mathrm{Bernoulli}(p)$, a single coverage constraint with threshold
$\tau\in(0,1)$, and $a_{1t}=w_t=1$.
The fixed intervals $C_{11t}=[0,1/2]$ and $C_{21t}=[0,1]$
have width costs $1/2<1$ and miscoverage probabilities $p$ and $0$.
The oracle is therefore $\{1\}$ when $p\le\tau$ and $\{2\}$
when $p>\tau$: the cost ordering is known, but the identity of the
oracle depends on whether the cheaper candidate is feasible.

\begin{proposition}
\label{prop:certification-limits}
Let $T_{\mathrm{cert}}$ be a stopping time for the natural filtration
of $(B_t)_{t\ge1}$, where $T_{\mathrm{cert}}<\infty$ means that
$p\le\tau$ is certified.
For any $\eta\in(0,1)$,
\[
\sup_{p>\tau}\mathbb P_p(T_{\mathrm{cert}}<\infty)\le\eta
\quad\Longrightarrow\quad
\mathbb P_\tau(T_{\mathrm{cert}}<\infty)\le\eta.
\]
\end{proposition}

For any model-set rule $(\widehat{\mathcal S}_t)_{t\ge1}$ adapted
to this filtration and satisfying (\ref{eq:ccsmcsdef}) for every $p\in(0,1)$,
the first exclusion of candidate~2 is a safe certification time
with $\eta=\delta$.
Consequently,
\[
\mathbb P_\tau\!\left(
\exists t\ge1:\widehat{\mathcal S}_t=\{1\}
\right)\le\delta.
\]
Thus, a unique oracle with strictly lower observed cost does not
guarantee high-probability singleton identification under uniform
oracle coverage.
Appendix~\ref{app:certification-limits} proves these claims and gives
a finite-sample lower bound: for fixed threshold and error levels,
high-probability identification requires at least order
$(\tau-p)^{-2}$ observations as $p\uparrow\tau$.

The prespecified tolerance in (\ref{eq:tolerance}) creates a strict feasibility margin:
if $p\le\alpha_r$ and $\tau_r=\alpha_r+\varepsilon_r$ with
$\varepsilon_r>0$, then $p-\tau_r\le-\varepsilon_r$.
Under persistent exposure, this relaxed constraint is eventually
certified on the joint confidence event.

\section{Experiments}
\label{sec:experi}
\subsection{Simulations}
To evaluate the validity and power of CC-SMCS, we design three simulations that share the same variation process and point predictor but differ in oracle dynamics. These settings allow us to assess both oracle retention and adaptation to time-varying oracle models. Each simulation uses $T=3000$ observations and is replicated N=1,000 times.

\subsubsection{Data Generating Process}
\label{subsub:generate}
To obtain the population oracle for evaluating CC-SMCS, we directly generate the point forecast as $\widehat{Y}_t=Y_{t-1}$ and the outcome as $Y_t=Y_{t-1}+\sigma_tZ_t$, where $Z_t\overset{\mathrm{iid}}{\sim}\mathcal{N}(0,1)$. Following \citet{hansen2011mcs}, we generate $H_t=-\frac{\phi}{2(1+\phi)}+\phi H_{t-1}+\sqrt{\phi}\varepsilon_t$, where $\varepsilon_t\overset{\mathrm{iid}}{\sim}\mathcal{N}(0,1)$ and $\phi\in(0,1)$, with stationary initialization $H_0\sim\mathcal{N}\left(-\frac{\phi}{2(1-\phi^2)},\frac{\phi}{1-\phi^2}\right)$. We set $\phi=0.5$. The predictable baseline variation for forecasting period $t$ is defined as $\sigma_t^{0}=\exp\{H_{t-1}\}$.

Let $h_t=H_{t-1}+\frac{\phi}{2(1-\phi^2)}$, and $\sigma_{t}=\sigma^{0}_{t}\exp\{l_{t}|h_{t}|\}$. We design six methods. For M2--M6, we construct the candidate prediction intervals as
\[
C_{mr,t}=\left[\widehat{Y}_t-z\sigma_{mr,t},\widehat{Y}_t+z\sigma_{mr,t}\right],\qquad \sigma_{mr,t}=\sigma_t^{0}\exp\{\theta_{mr}|h_t|\}
\]
For M1, $C_{1r,t}=\left[\widehat{Y}_t+\kappa\sigma_{4r,t}-z\sigma_{4r,t},\widehat{Y}_t+\kappa\sigma_{4r,t}+z\sigma_{4r,t}\right]$,
$\kappa=0.5$, where we instead use a shifted center interval so that M1 has the same interval width as M4 but exhibits undercoverage due to the shifted center. 

For M2--M6, we set $(\theta_{2r},\ldots,\theta_{6r})=\xi(-2,-1,0,1,2)$, with $\xi=0.15$ and $z=\Phi^{-1}(0.975)$. In Simulation~1, M2 and M3 undercover, M4 is calibrated, and M5 and M6 overcover with wider intervals.
Under the true scale $\sigma_t$, the conditional miscoverage probability of a centered candidate is
\[
P\left\{Y_t\notin C_{mr,t}\mid\mathcal{G}_{t-1}\right\}=2\left[1-\Phi\left(z\frac{\sigma_{mr,t}}{\sigma_t}\right)\right].
\]

\paragraph{Simulation 1 (Stationary).} We set $l_t=0$ for all $t$, so M4 lies exactly on the coverage boundary and is the fixed oracle.

\paragraph{Simulation 2 (Structural Change).} We introduce a permanent structural break after $t_0=400$ by setting:
\[
l_t=\begin{cases}0, & t\leq 400,\\1.75\xi, & t>400.\end{cases}
\]
 
\paragraph{Simulation 3 (Linear Drift).} We let $l_t=2\xi(t-1)/(T-1)$, $t=1,\ldots, T$, so that the true forecast-error scale changes gradually over time. 

Further details of the simulation designs are provided in Appendix~\ref{sec:detail_sim}.

\subsubsection{Simulation Results}
We evaluate finite-horizon oracle retention using
\[
P_{\mathrm{TU}}=\frac{1}{N}\sum_{n=1}^{N}\mathbf{1}\left\{\mathcal M^{\star}_{n,t}\subseteq\widehat{\mathcal M}^{\mathrm{proj}}_{n,t}\text{ for all } t\leq T\right\},
\]
which estimates the retention probability in Theorem~\ref{thm:main} over the evaluation horizon. We also report the mean confidence radius across methods, $\bar{\mathrm{rad}}_t$, at three time points to illustrate contraction. For the boundary method $M_\tau=M_4$ in Simulation~1, we report the fraction of runs with certification by $T$,
\[
\widehat p_{\tau,T}=\frac{1}{N}\sum_{n=1}^{N}\mathbf{1}\left\{\exists\,t\leq T:M_\tau\in\widehat{\mathcal V}^{\mathrm{cert}}_{n,t}\right\}.
\]
Table~\ref{tab:resu_sim} shows oracle-retention rates close to one, decreasing mean radii, and rare boundary certification. The population oracle set is nonempty at every evaluated time in all replications.
\begin{table}[!ht]
\centering
\caption{Simulation results}
\label{tab:resu_sim}
\begin{threeparttable}
\resizebox{0.6\linewidth}{!}{%
\begin{tabular}{lcccccc}
\toprule
Simulations & $\delta$ & $P_{TU}$ & $\bar{\mathrm{rad}}_{t_{0.25}}$ & $\bar{\mathrm{rad}}_{t_{0.50}}$ & $\bar{\mathrm{rad}}_{t_{0.75}}$ & $\widehat p_{\tau,T}$ \\
\midrule
\multirow{2}{*}{Simulation 1}
& 0.10 & 1.000 & 0.0361 & 0.0234 & 0.0186 & 0.001 \\
& 0.25 & 1.000& 0.0326 & 0.0211 & 0.0168 & 0.001 \\
\midrule
\multirow{2}{*}{Simulation 2}
& 0.10 & 1.000 & 0.0402 & 0.0283 & 0.0231 & -- \\
& 0.25 & 1.000 & 0.0363 & 0.0255 & 0.0209 & -- \\
\midrule
\multirow{2}{*}{Simulation 3}
& 0.10 & 0.999 & 0.0372 & 0.0252 & 0.0209 & -- \\
& 0.25 & 0.999 & 0.0336 & 0.0227 & 0.0188 & -- \\
\bottomrule
\end{tabular}%
}
\end{threeparttable}
\vspace{2pt}
\parbox{\linewidth}{%
\footnotesize
\textit{Note:} $t_q=qT$ for $q\in\{0.25,0.50,0.75\}$.
$\widehat p_{\tau,T}$ is reported only for Simulation~1, where
$M_4$ remains on the coverage boundary throughout the evaluation.
}
\end{table}

\begin{figure}[!ht]
\centering
\begin{subfigure}{\linewidth}
\centering
\includegraphics[width=\linewidth]{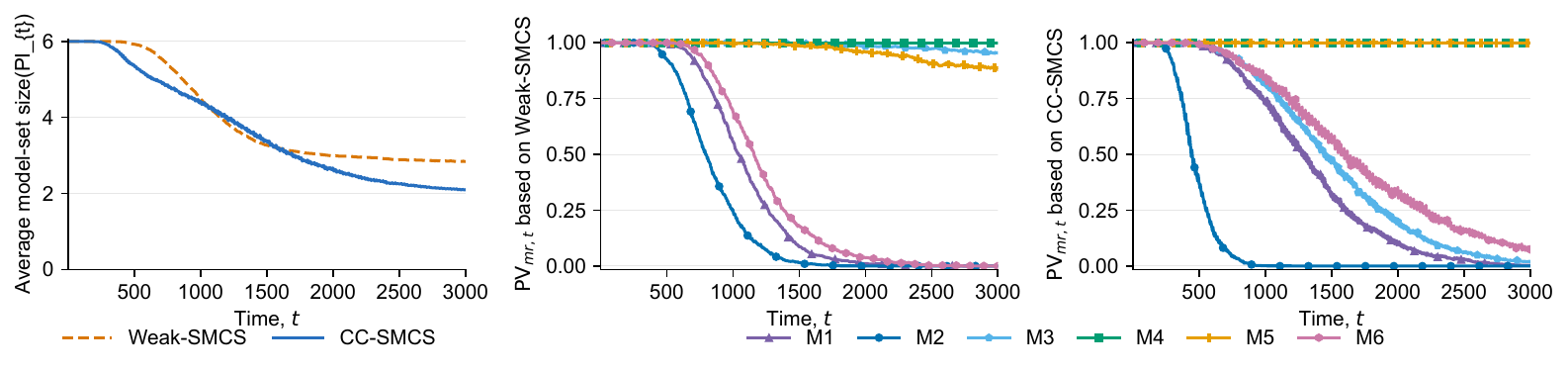}
\caption{Simulation 1}
\label{fig:sim1_010}
\end{subfigure}
\begin{subfigure}{\linewidth}
\centering
\includegraphics[width=\linewidth]{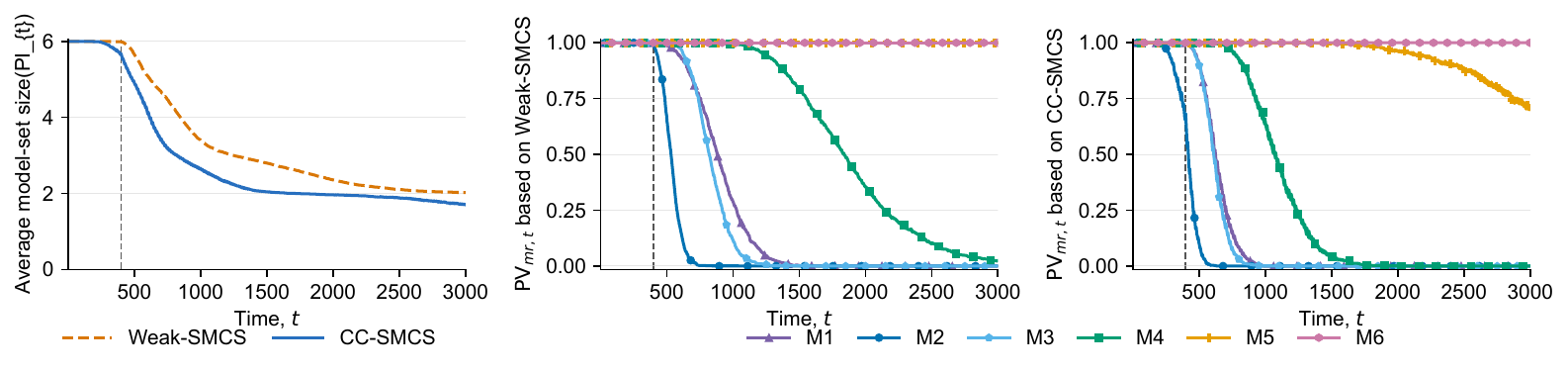}
\caption{Simulation 2}
\label{fig:sim2_010}
\end{subfigure}
\begin{subfigure}{\linewidth}
\centering
\includegraphics[width=\linewidth]{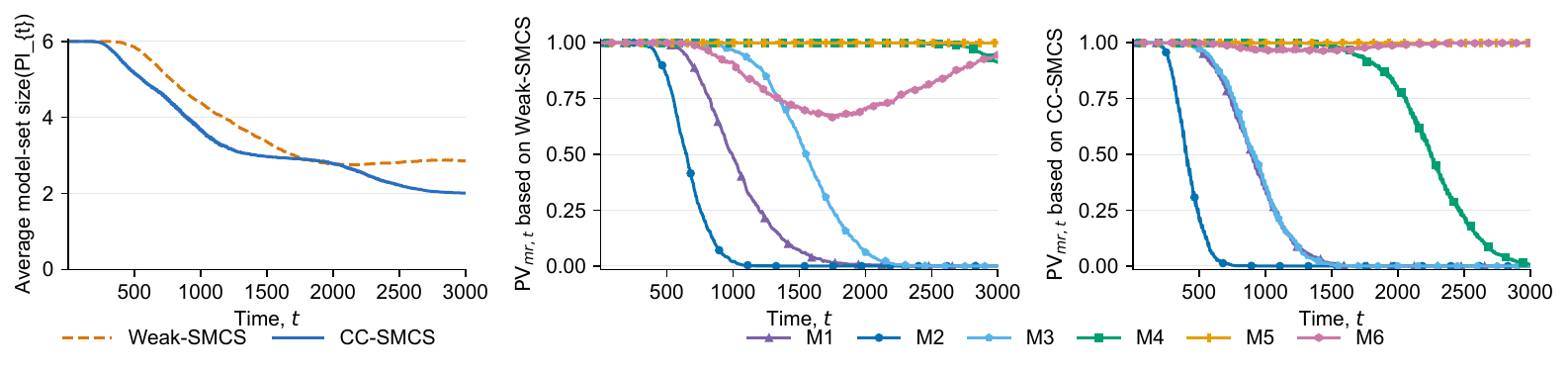}
\caption{Simulation 3}
\label{fig:sim3_010}
\end{subfigure}
\caption{The performance of simulations with $\delta=0.10$. The vertical dashed line in Figure~\ref{fig:sim2_010} indicates the structural break.}
\label{fig:sim_010}
\end{figure}

We summarize the model-set paths by the inclusion frequency $\mathrm{PV}_{m,t}$ and mean set size $\mathrm{PI}_t$:
\begin{equation}
\mathrm{PV}_{m,t}=\frac{1}{N}\sum_{n=1}^{N}\mathbf{1}\left\{m\in\widehat{\mathcal M}^{\mathrm{proj}}_{n,t}\right\},\qquad\mathrm{PI}_t=\frac{1}{N}\sum_{n=1}^{N}\left|\widehat{\mathcal M}^{\mathrm{proj}}_{n,t}\right|.
\end{equation}
Figure~\ref{fig:sim_010} compares CC-SMCS with Weak-SMCS \citep{arnold2026sequential} at $\delta=0.10$. Benchmark details and additional results are given in Appendices~\ref{sub:weak_smcs} and~\ref{sub:addi_re}, respectively.

In Simulation~1, CC-SMCS reduces the set while retaining $M_4$. The continued inclusion of $M_5$ is consistent with the difficulty of certifying the boundary oracle $M_4$. In Simulations~2 and~3, the oracle moves from $M_4$ to $M_5$ and then $M_6$ after a structural break or under gradual drift. CC-SMCS excludes $M_4$ as evidence of undercoverage accumulates. Under gradual drift, the inclusion frequency of $M_6$ first decreases and then increases as it becomes the constrained oracle. Weak-SMCS follows different paths because it compares standardized Winkler-score differences rather than interval width subject to coverage constraints.

\subsection{Empirical Study}
In this section, we present empirical applications of CC-SMCS using four real-world time-series datasets: French Electricity Price, Electricity Load, Weather, and Wind Power, which were used by \citet{zaffran2022adaptive}, \citet{hong2016probabilistic}, \citet{wu2023timesnet}, and \citet{xu2021enbpi}, respectively. The Electricity Load and Weather datasets were also utilized by \citet{lin2022conformal} and \citet{chen2024conformalized}, respectively. For each dataset, we employ DLinear \citep{zeng2023transformers}, PatchTST \citep{nie2023patchtst}, and TimesNet \citep{wu2023timesnet} as forecasting backbones, then construct prediction intervals by classical split conformal prediction (Split CP) \citep{lei2018distribution}, weighted conformal prediction (WCP) \citep{barber2023conformal}, Adaptive Conformal Inference (ACI) \citep{gibbs2021adaptive}, Aggregated ACI (AgACI) \citep{zaffran2022adaptive}, conformalized quantile regression (CQR) \citep{romano2019cqr}, and EnbPI \citep{xu2021enbpi}. Implementation details are provided in Appendix~\ref{sec:add_exper}.
\begin{figure}[!ht]
\centering
\begin{subfigure}{\linewidth}
\centering
\includegraphics[width=\linewidth]{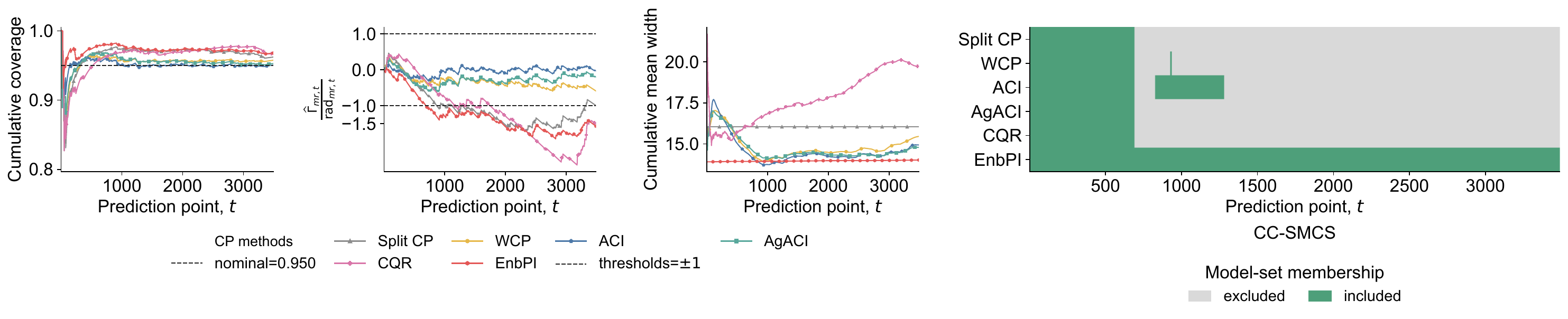}
\caption{French Electricity Price}
\label{fig:france_010}
\end{subfigure}
\begin{subfigure}{\linewidth}
\centering
\includegraphics[width=\linewidth]{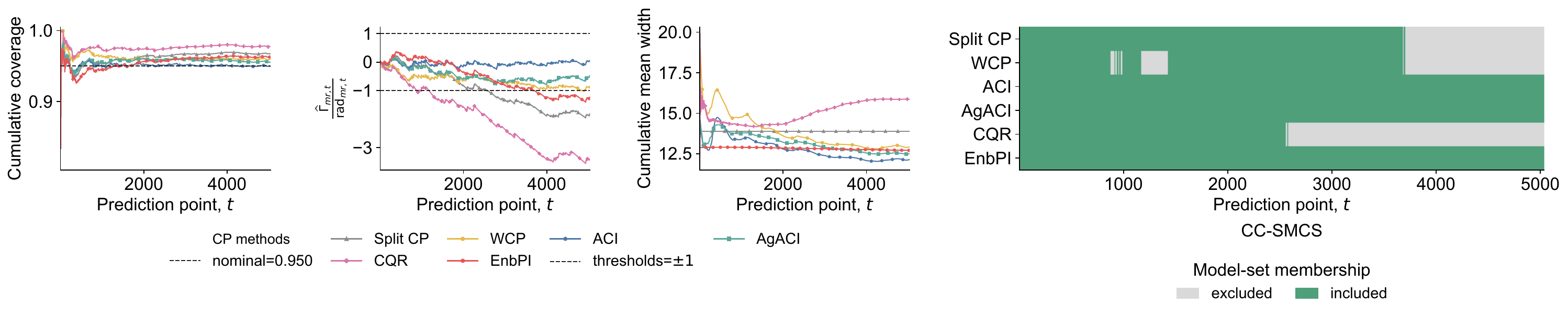}
\caption{Electricity Load}
\label{fig:load_010}
\end{subfigure}

\vspace{0.1cm}

\begin{subfigure}{\linewidth}
\centering
\includegraphics[width=\linewidth]{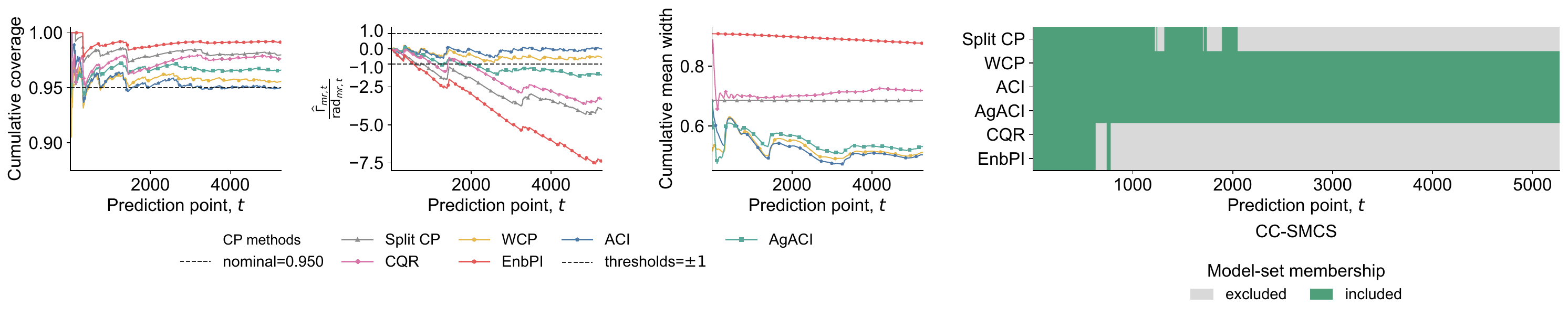}
\caption{Weather}
\label{fig:weather_010}
\end{subfigure}

\vspace{0.1cm}

\begin{subfigure}{\linewidth}
\centering
\includegraphics[width=\linewidth]{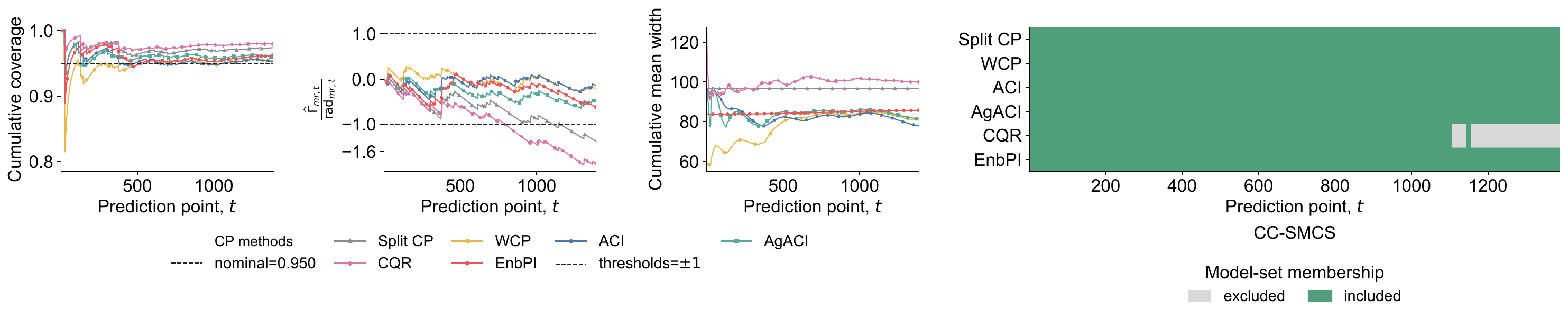}
\caption{Wind Power}
\label{fig:wind_010}
\end{subfigure}

\caption{Empirical results for DLinear at 95\% nominal coverage with $\delta=0.10$.}
\label{fig:dlinear_010}
\end{figure}

Figure~\ref{fig:dlinear_010} presents the DLinear results for French Electricity Price Electricity Load, Weather and Wind Power. Additional results are in Appendix~\ref{sub:add_re}. The constrained oracle is unobservable, so coverage and width help explain the reported model sets. The ratio $\widehat{\Gamma}_{mr,t}/\mathrm{rad}_{mr,t}$, with $\widehat{\Gamma}_{mr,t}=S_{mr,t}/A_{r,t}$, summarizes the coverage evidence: values above $1$ rule out feasibility, while values at most
$-1$ certify it. On four datasets, CQR is excluded despite high coverage when $Q_{\mathrm{CQR},t}>q_t^{\mathrm{cert}}$, illustrating how CC-SMCS separates coverage certification from efficiency comparison. The re-entry of ACI in French Electricity Price, WCP in Electricity Load and Split CP in Weather also reflect the time-varying nature of CC-SMCS.

\section{Conclusion}
\label{sec:conclusion}
We introduced \CCSMCS\ to compare conformal pipelines under coverage constraints. The method separates feasibility from efficiency and provides finite-sample, time-uniform oracle coverage. Our theory gives conditions for oracle recovery and identifies limits at the coverage boundary. Simulations and real data examples illustrate how the model sets respond to changes in coverage and interval width.

\bibliographystyle{model2-names}
\bibliography{references}

\clearpage

\newpage
\appendix
\section{Notation and Algorithm}
\label{app:notation}

Table~\ref{tab:notation} summarizes the main objects. The distinction between the three reported sets is essential: $\widehat\Vset_t^{\cert}$ is an inner confidence set for feasibility, $\widehat\Vset_t^{\pos}$ is an outer confidence set for feasibility, and $\widehat\M_t^{\proj}$ is an outer confidence set for the constrained optimizer.

\begin{table}[t]
\centering
\caption{Main notation.}
\label{tab:notation}
\small
\begin{tabular}{ll}
\toprule
Symbol & Meaning \\
\midrule
$C_{mrt}$ & Prediction set from method $m$ under constraint $r$ at round $t$ \\
$B_{mrt}$ & Miscoverage indicator $\1\{Y_t\notin C_{mrt}\}$ \\
$a_{rt}$ & Predictable exposure or subgroup weight \\
$\Gamma_{mr,t}$ & Prefix-average conditional excess miscoverage \\
$Q_{m,t}$ & Observed prefix-average set cost \\
$\Vset_t^\star$ & True feasible method set \\
$\M_t^\star$ & Minimum-cost subset of $\Vset_t^\star$ \\
$\widehat\Vset_t^{\cert}$ & Methods certified feasible \\
$\widehat\Vset_t^{\pos}$ & Methods not ruled infeasible \\
$\widehat\M_t^{\proj}$ & Coverage-constrained sequential model confidence set \\
\bottomrule
\end{tabular}
\end{table}

The main event in Theorem~\ref{thm:main} simultaneously implies
\[
\widehat\Vset_t^{\cert}\subseteq\Vset_t^\star\subseteq\widehat\Vset_t^{\pos},\qquad \M_t^\star\subseteq\widehat\M_t^{\proj}\quad\text{for every }t.
\]
It does not imply $\widehat\M_t^{\proj}\subseteq\Vset_t^\star$.

Algorithm~\ref{alg:ccsmcs} summarizes the \CCSMCS\ procedure. If no method is certified, then $q_t^{\cert}=+\infty$ and the output equals $\widehat\Vset_t^{\pos}$. A positive practical tolerance $\varepsilon_r$ creates a strict feasibility margin for methods satisfying the nominal target $p\leq\alpha_r$, which permits eventual certification under persistent exposure. Section~\ref{sec:coverage-boundary} gives the precise statement.
\begin{algorithm}[t]
\caption{Empirical-Bernstein CC-SMCS}
\label{alg:ccsmcs}
\begin{algorithmic}[1]

\REQUIRE Candidate prediction sets $\{C_{mrt}\}_{m\in[M],r\in[R]}$, thresholds $\{\tau_r\}_{r\in[R]}$, predictable exposures $\{a_{rt}\}_{r\in[R]}$, efficiency costs $\{Q_{m,t}\}_{m\in[M]}$, confidence level $\delta$, allocations $\{\pi_{mr}\}$ with $\sum_{m,r}\pi_{mr}=1$, and mixture parameter $\rho>0$.

\STATE Set $A_{r,0}=0$, $S_{mr,0}=0$, and $V_{mr,0}=0$ for all $m,r$.
\STATE Set $\eta_{mr}=\delta\pi_{mr}$ for all $m,r$.

\FOR{$t=1,2,\ldots$}

\FOR{$r=1,\ldots,R$}

\FOR{$m=1,\ldots,M$}

\IF{$A_{r,t-1}=0$}
\STATE Set $\widehat B_{mrt}=\tau_r$.
\ELSE
\STATE Set $\widehat B_{mrt}=\dfrac{\sum_{s=1}^{t-1}a_{rs}B_{mrs}}{A_{r,t-1}}$.
\ENDIF

\STATE Set $\widehat X_{mrt}=a_{rt}(\widehat B_{mrt}-\tau_r)$.

\ENDFOR

\ENDFOR

\STATE Observe $Y_t$.

\FOR{$r=1,\ldots,R$}

\STATE Update $A_{r,t}=A_{r,t-1}+a_{rt}$.

\FOR{$m=1,\ldots,M$}

\STATE Compute $B_{mrt}=\mathbf 1\{Y_t\notin C_{mrt}\}$ and $X_{mrt}=a_{rt}(B_{mrt}-\tau_r)$.
\STATE Update $S_{mr,t}=S_{mr,t-1}+X_{mrt}$.
\STATE Update $V_{mr,t}=V_{mr,t-1}+(X_{mrt}-\widehat X_{mrt})^2$.

\IF{$A_{r,t}>0$}
\STATE Compute $b_{\eta_{mr},\rho}(V_{mr,t})$ from $\mathfrak m_\rho(b_{\eta_{mr},\rho}(V_{mr,t}),V_{mr,t})=2/\eta_{mr}$.
\STATE Set $L_{mr,t}=\dfrac{S_{mr,t}-b_{\eta_{mr},\rho}(V_{mr,t})}{A_{r,t}}$ and $U_{mr,t}=\dfrac{S_{mr,t}+b_{\eta_{mr},\rho}(V_{mr,t})}{A_{r,t}}$.
\ELSE
\STATE Set $L_{mr,t}=-\infty$ and $U_{mr,t}=+\infty$.
\ENDIF

\ENDFOR

\ENDFOR

\STATE Construct $\widehat{\Vset}_t^{\mathrm{pos}}=\{m:L_{mr,t}\le0\text{ for every }r\}$.
\STATE Construct $\widehat{\Vset}_t^{\mathrm{cert}}=\{m:U_{mr,t}\le0\text{ for every }r\}$.
\STATE Set $q_t^{\mathrm{cert}}=\min_{j\in\widehat{\Vset}_t^{\mathrm{cert}}}Q_{j,t}$, with $\min\varnothing:=+\infty$.
\STATE Construct $\widehat\M_t^{\proj}=\{m\in\widehat{\Vset}_t^{\mathrm{pos}}:Q_{m,t}\le q_t^{\mathrm{cert}}\}$.
\STATE Output $\widehat{\Vset}_t^{\mathrm{pos}}$, $\widehat{\Vset}_t^{\mathrm{cert}}$, and $\widehat\M_t^{\proj}$.

\ENDFOR

\end{algorithmic}
\end{algorithm}
\section{Proofs}
\label{app:proofs}

\subsection{Empirical-Bernstein confidence sequence}
\label{app:mixtureproof}

For $\rho>0$ and $z,v\geq0$, define the gamma-exponential mixture function
\begin{equation}
\mathfrak m_\rho(z,v)=\frac{\rho^\rho e^{z+v}\gamma(v+\rho,z+v+\rho)}{\gamma(\rho,\rho)(z+v+\rho)^{v+\rho}},
\label{eq:mixture_function}
\end{equation}
where $\gamma(a,h)=\int_0^h y^{a-1}e^{-y}\,dy$ is the lower incomplete gamma function. The details of the selection of $\rho$ are provided in Appendix~\ref{sec:select_rho}.
For a two-sided error probability $\eta\in(0,1)$, define
\begin{equation}
b_{\eta,\rho}(v)=\inf\left\{z\geq0:\mathfrak m_\rho(z,v)\geq\frac{2}{\eta}\right\}.
\label{eq:boundary}
\end{equation}
The threshold $2/\eta$ allocates error probability $\eta/2$ to each tail.

\begin{lemma}\label{lem:hoeffding}
Let $\mu_{mrt}:=\E[X_{mrt}\mid\G_{t-1}]$, $D_{mrt}:=X_{mrt}-\mu_{mrt}$, and $\Delta_{mrt}:=(X_{mrt}-\widehat X_{mrt})^2$. Suppose that $a_{rt}\in[0,1]$ is $\G_{t-1}$-measurable, $\tau_r\in[0,1]$, $B_{mrt}\in\{0,1\}$, and $\widehat B_{mrt}\in[0,1]$ is $\G_{t-1}$-measurable. Then
\[
\E[D_{mrt}\mid\G_{t-1}]=0.
\label{eq:conditional_mean_zero}
\]
Moreover, for every $\lambda\in[0,1)$, with $\psi_E(\lambda):=-\log(1-\lambda)-\lambda$,
\begin{equation}
\E\left[\exp\{\lambda D_{mrt}-\psi_E(\lambda)\Delta_{mrt}\}\mid\G_{t-1}\right]\leq1.
\label{eq:conditionalhoeffding}
\end{equation}
\end{lemma}

\begin{proof}
By definition,
\[
\E[D_{mrt}\mid\G_{t-1}]=\E[X_{mrt}-\mu_{mrt}\mid\G_{t-1}]=\E[X_{mrt}\mid\G_{t-1}]-\mu_{mrt}=0.
\]
Since $B_{mrt}\in\{0,1\}$, $\widehat B_{mrt}\in[0,1]$, and $a_{rt}\in[0,1]$, both $X_{mrt}=a_{rt}(B_{mrt}-\tau_r)$ and $\widehat X_{mrt}=a_{rt}(\widehat B_{mrt}-\tau_r)$ belong to $[-\tau_r,1-\tau_r]$. Hence
\[
X_{mrt}-\widehat X_{mrt}=a_{rt}(B_{mrt}-\widehat B_{mrt})\in[-a_{rt},a_{rt}]\subseteq[-1,1]
\]
almost surely. Following the self-normalization argument in the proof of Theorem~4 of \citet{howard2021time}, for every $\lambda\in[0,1)$ and $x\geq-1$,
\begin{equation}
\exp\{\lambda x-\psi_E(\lambda)x^2\}\leq1+\lambda x.
\label{eq:fan_inequality}
\end{equation}
Applying \eqref{eq:fan_inequality} to $x=X_{mrt}-\widehat X_{mrt}$ gives
\[
\exp\{\lambda(X_{mrt}-\widehat X_{mrt})-\psi_E(\lambda)(X_{mrt}-\widehat X_{mrt})^2\}\leq1+\lambda(X_{mrt}-\widehat X_{mrt}).
\label{eq:fan_applied}
\]
Multiplying both sides by $\exp\{\lambda(\widehat X_{mrt}-\mu_{mrt})\}$ yields
\[
\exp\{\lambda(X_{mrt}-\mu_{mrt})-\psi_E(\lambda)(X_{mrt}-\widehat X_{mrt})^2\}\leq\exp\{\lambda(\widehat X_{mrt}-\mu_{mrt})\}\{1+\lambda(X_{mrt}-\widehat X_{mrt})\}.
\label{eq:self_normalization_step}
\]
Taking conditional expectations with respect to $\G_{t-1}$ and using the $\G_{t-1}$-measurability of $\widehat X_{mrt}$ gives
\[
\begin{aligned}
&\E\left[\exp\{\lambda(X_{mrt}-\mu_{mrt})-\psi_E(\lambda)(X_{mrt}-\widehat X_{mrt})^2\}\mid\G_{t-1}\right]\notag\\
&\leq\exp\{\lambda(\widehat X_{mrt}-\mu_{mrt})\}\{1+\lambda(\mu_{mrt}-\widehat X_{mrt})\}\notag\\
&\leq\exp\{\lambda(\widehat X_{mrt}-\mu_{mrt})\}\exp\{\lambda(\mu_{mrt}-\widehat X_{mrt})\}=1,
\label{eq:one_step_bound}
\end{aligned}
\]
where the last inequality uses $1-y\leq e^{-y}$.
\end{proof}

\begin{lemma}
\label{lem:mixed_bound}
Let
\[
Z_{mr,t}:=\sum_{s=1}^tD_{mrs},\qquad \Delta_{mrt}:=(X_{mrt}-\widehat X_{mrt})^2,\qquad V_{mr,t}:=\sum_{s=1}^t\Delta_{mrs}=\sum_{s=1}^t(X_{mrs}-\widehat X_{mrs})^2.
\]
For every $\eta\in(0,1)$ and $\rho>0$, with $b_{\eta,\rho}$ defined in \eqref{eq:boundary},
\begin{equation}
\Pp\left(\forall t\geq1:\ |Z_{mr,t}|\leq b_{\eta,\rho}(V_{mr,t})\right)\geq1-\eta.
\label{eq:mixturelemma}
\end{equation}
\end{lemma}

\begin{proof}
Fix $\lambda\in[0,1)$ and define
\[
M_t(\lambda):=\exp\{\lambda Z_{mr,t}-\psi_E(\lambda)V_{mr,t}\}.
\label{eq:a_mart}
\]
Because $D_{mrt}$ and $\Delta_{mrt}$ are $\G_t$-measurable, $M_t(\lambda)$ is $\G_t$-measurable. By Lemma~\ref{lem:hoeffding},
\[
\E\left[\exp\{\lambda D_{mrt}-\psi_E(\lambda)\Delta_{mrt}\}\mid\G_{t-1}\right]\leq1.
\]
Using $Z_{mr,t}=Z_{mr,t-1}+D_{mrt}$ and $V_{mr,t}=V_{mr,t-1}+\Delta_{mrt}$, we have
\[
\begin{aligned}
M_t(\lambda)&=\exp\{\lambda Z_{mr,t-1}+\lambda D_{mrt}-\psi_E(\lambda)V_{mr,t-1}-\psi_E(\lambda)\Delta_{mrt}\}\notag\\
&=M_{t-1}(\lambda)\exp\{\lambda D_{mrt}-\psi_E(\lambda)\Delta_{mrt}\}.
\label{eq:martingale_factorization}
\end{aligned}
\]
Taking conditional expectations and using that $M_{t-1}(\lambda)$ is $\G_{t-1}$-measurable gives
\begin{align}
\E[M_t(\lambda)\mid\G_{t-1}]&=M_{t-1}(\lambda)\E\left[\exp\{\lambda D_{mrt}-\psi_E(\lambda)\Delta_{mrt}\}\mid\G_{t-1}\right]\notag\\
&\leq M_{t-1}(\lambda).
\label{eq:supermartingale_property}
\end{align}
Since $M_t(\lambda)\geq0$, taking expectations and iterating the preceding inequality yields
\[
\E[M_t(\lambda)]\leq\E[M_{t-1}(\lambda)]\leq\cdots\leq\E[M_0(\lambda)]=1.
\]
Thus $M_t(\lambda)$ is integrable, and $(M_t(\lambda))_{t\geq0}$ is a nonnegative $(\G_t)_{t\geq0}$-supermartingale with $M_0(\lambda)=1$.

Since
\[
e^{\lambda z-\psi_E(\lambda)v}=e^{\lambda(z+v)}(1-\lambda)^v,
\label{eq:subexp_kernel}
\]
following the gamma-exponential mixture construction in Proposition~9 of \citet{howard2021time}, for $\rho>0$ define
\[
f_\rho(\lambda)=\frac{\rho^\rho}{\gamma(\rho,\rho)}(1-\lambda)^{\rho-1}e^{-\rho(1-\lambda)},\qquad 0\leq\lambda<1.
\label{eq:density}
\]
The factor $\rho^\rho/\gamma(\rho,\rho)$ is the normalizing constant, since
\[
\int_0^1(1-\lambda)^{\rho-1}e^{-\rho(1-\lambda)}\,d\lambda=\rho^{-\rho}\gamma(\rho,\rho).
\label{eq:density_normalization}
\]
Define the mixture process
\[
M_t:=\int_0^1M_t(\lambda)f_\rho(\lambda)\,d\lambda.
\label{eq:mixture_process}
\]
Since the integrand is nonnegative, Tonelli's theorem gives
\begin{align*}
\E[M_t]&=\E\left[\int_0^1M_t(\lambda)f_\rho(\lambda)\,d\lambda\right]\\
&=\int_0^1\E[M_t(\lambda)]f_\rho(\lambda)\,d\lambda\\
&\leq\int_0^1f_\rho(\lambda)\,d\lambda=1.
\end{align*}
Thus $M_t$ is integrable. By conditional Tonelli's theorem and \eqref{eq:supermartingale_property},
\begin{align*}
\E[M_t\mid\G_{t-1}]&=\int_0^1\E[M_t(\lambda)\mid\G_{t-1}]f_\rho(\lambda)\,d\lambda\\
&\leq\int_0^1M_{t-1}(\lambda)f_\rho(\lambda)\,d\lambda=M_{t-1}.
\end{align*}
Moreover,
\[
M_0=\int_0^1M_0(\lambda)f_\rho(\lambda)\,d\lambda=\int_0^1f_\rho(\lambda)\,d\lambda=1.
\]
Hence $(M_t)_{t\geq0}$ is a nonnegative $(\G_t)_{t\geq0}$-supermartingale. By Ville's inequality,
\begin{equation}
\Pp\left(\exists t\geq1:\ M_t\geq\frac{2}{\eta}\right)\leq\frac{\eta}{2}.
\label{eq:ville_mixture}
\end{equation}
For $z\geq0$ and $v\geq0$, the corresponding mixture integral is
\[
\int_0^1e^{\lambda z-\psi_E(\lambda)v}f_\rho(\lambda)\,d\lambda=\frac{\rho^\rho}{\gamma(\rho,\rho)}\int_0^1e^{\lambda(z+v)}(1-\lambda)^{v+\rho-1}e^{-\rho(1-\lambda)}\,d\lambda.
\]
Set $y=(1-\lambda)(z+v+\rho)$. Then
\begin{align*}
\int_0^1e^{\lambda z-\psi_E(\lambda)v}f_\rho(\lambda)\,d\lambda
&=\frac{\rho^\rho e^{z+v}}{\gamma(\rho,\rho)}\int_0^1(1-\lambda)^{v+\rho-1}e^{-(1-\lambda)(z+v+\rho)}\,d\lambda\\
&=\frac{\rho^\rho e^{z+v}}{\gamma(\rho,\rho)(z+v+\rho)^{v+\rho}}\int_0^{z+v+\rho}y^{v+\rho-1}e^{-y}\,dy\\
&=\frac{\rho^\rho e^{z+v}\gamma(v+\rho,z+v+\rho)}{\gamma(\rho,\rho)(z+v+\rho)^{v+\rho}}\\
&=\mathfrak m_\rho(z,v),
\end{align*}
which is the expression used in \eqref{eq:mixture_function}.
For fixed $v\geq0$,
\[
\frac{\partial}{\partial z}\mathfrak m_\rho(z,v)=\int_0^1\lambda e^{\lambda z-\psi_E(\lambda)v}f_\rho(\lambda)\,d\lambda>0.
\label{eq:monotone}
\]
Thus $\mathfrak m_\rho(z,v)$ is continuous and strictly increasing in $z\geq0$. By the definition of $b_{\eta,\rho}(v)$, if $Z_{mr,t}\geq b_{\eta,\rho}(V_{mr,t})$, then $Z_{mr,t}\geq0$ and $M_t\geq2/\eta$. Therefore \eqref{eq:ville_mixture} gives
\begin{equation}
\Pp\left(\exists t\geq1:\ Z_{mr,t}\geq b_{\eta,\rho}(V_{mr,t})\right)\leq\frac{\eta}{2}.
\label{eq:upper_tail}
\end{equation}
For the lower tail, consider the reflected variables
\[
X_{mrt}^{-}:=-X_{mrt},\qquad \widehat X_{mrt}^{-}:=-\widehat X_{mrt},\qquad \mu_{mrt}^{-}:=-\mu_{mrt}.
\]
Then
\[
X_{mrt}^{-}-\widehat X_{mrt}^{-}=-(X_{mrt}-\widehat X_{mrt})\in[-1,1],
\]
and
\[
(X_{mrt}^{-}-\widehat X_{mrt}^{-})^2=(X_{mrt}-\widehat X_{mrt})^2.
\]
Moreover,
\[
X_{mrt}^{-}-\mu_{mrt}^{-}=-(X_{mrt}-\mu_{mrt}),
\]
so the corresponding centered cumulative process and empirical intrinsic time satisfy
\[
Z_{mr,t}^{-}:=\sum_{s=1}^t(X_{mrs}^{-}-\mu_{mrs}^{-})=-Z_{mr,t},\qquad V_{mr,t}^{-}:=\sum_{s=1}^t(X_{mrs}^{-}-\widehat X_{mrs}^{-})^2=V_{mr,t}.
\]
Applying \eqref{eq:fan_inequality} to $X_{mrt}^{-}-\widehat X_{mrt}^{-}$ and repeating the preceding one-step argument gives, for every $\lambda\in[0,1)$,
\begin{equation}
M_t^{-}(\lambda):=\exp\{-\lambda Z_{mr,t}-\psi_E(\lambda)V_{mr,t}\}
\label{eq:reflected_mart}
\end{equation}
as a nonnegative $(\G_t)_{t\geq0}$-supermartingale. Mixing over the same density $f_\rho$ gives
\begin{equation}
M_t^{-}:=\int_0^1M_t^{-}(\lambda)f_\rho(\lambda)\,d\lambda.
\label{eq:reflected_mixture}
\end{equation}
By the same Tonelli argument, $(M_t^{-})_{t\geq0}$ is a nonnegative supermartingale with $M_0^{-}=1$. Hence Ville's inequality gives
\begin{equation}
\Pp\left(\exists t\geq1:\ M_t^{-}\geq\frac{2}{\eta}\right)\leq\frac{\eta}{2}.
\label{eq:ville_lower}
\end{equation}
If $-Z_{mr,t}\geq b_{\eta,\rho}(V_{mr,t})$, then $-Z_{mr,t}\geq0$ and $M_t^{-}\geq2/\eta$. Therefore
\begin{equation}
\Pp\left(\exists t\geq1:\ -Z_{mr,t}\geq b_{\eta,\rho}(V_{mr,t})\right)\leq\frac{\eta}{2}.
\label{eq:lower_tail}
\end{equation}
A union bound over \eqref{eq:upper_tail} and \eqref{eq:lower_tail} yields
\begin{equation}
\Pp\left(\forall t\geq1:\ |Z_{mr,t}|\leq b_{\eta,\rho}(V_{mr,t})\right)\geq1-\eta.
\label{eq:two_sided_cs}
\end{equation}
This proves \eqref{eq:mixturelemma}.
\end{proof}

\begin{lemma}
\label{lem:gamma_interval}
Choose fixed weights $\pi_{mr}>0$ with $\sum_{m,r}\pi_{mr}=1$ and set $\eta_{mr}=\delta\pi_{mr}$. For $A_{r,t}>0$, let $L_{mr,t}$ and $U_{mr,t}$ be the lower and upper bounds defined in \eqref{eq:LU}. If $A_{r,t}=0$, set $L_{mr,t}=-\infty$ and $U_{mr,t}=+\infty$. Then
\begin{equation}
\Pp\left(\forall t,m,r:\ L_{mr,t}\leq\Gamma_{mr,t}\leq U_{mr,t}\right)\geq1-\delta.
\end{equation}
\end{lemma}

\begin{proof}
For each $(m,r)$, apply Lemma~\ref{lem:mixed_bound} with error level $\eta_{mr}=\delta\pi_{mr}$. A union bound over all method--constraint pairs gives
\begin{align}
&\Pp\left(\forall m,r,\ \forall t\geq1:\ |Z_{mr,t}|\leq b_{\eta_{mr},\rho}(V_{mr,t})\right)\notag\\
&\qquad\geq1-\sum_{m=1}^M\sum_{r=1}^R\Pp\left(\exists t\geq1:\ |Z_{mr,t}|>b_{\eta_{mr},\rho}(V_{mr,t})\right)\notag\\
&\qquad\geq1-\sum_{m=1}^M\sum_{r=1}^R\eta_{mr}=1-\delta.
\label{eq:joint_boundary}
\end{align}

On this joint event, if $A_{r,t}>0$, then by the definition of $\Gamma_{mr,t}$,
\[
Z_{mr,t}=S_{mr,t}-\sum_{s=1}^t\mu_{mrs}=S_{mr,t}-A_{r,t}\Gamma_{mr,t}.
\]
Hence
\[
\left|S_{mr,t}-A_{r,t}\Gamma_{mr,t}\right|\leq b_{\eta_{mr},\rho}(V_{mr,t}),
\]
which is equivalent to
\[
-b_{\eta_{mr},\rho}(V_{mr,t})\leq S_{mr,t}-A_{r,t}\Gamma_{mr,t}\leq b_{\eta_{mr},\rho}(V_{mr,t}).
\]
Rearranging gives
\[
S_{mr,t}-b_{\eta_{mr},\rho}(V_{mr,t})\leq A_{r,t}\Gamma_{mr,t}\leq S_{mr,t}+b_{\eta_{mr},\rho}(V_{mr,t}).
\]
Since $A_{r,t}>0$, division by $A_{r,t}$ preserves both inequalities:
\[
\frac{S_{mr,t}-b_{\eta_{mr},\rho}(V_{mr,t})}{A_{r,t}}\leq\Gamma_{mr,t}\leq\frac{S_{mr,t}+b_{\eta_{mr},\rho}(V_{mr,t})}{A_{r,t}}.
\]
By the definition of the confidence bounds in \eqref{eq:LU}, this is exactly
\[
L_{mr,t}\leq\Gamma_{mr,t}\leq U_{mr,t}.
\]
If $A_{r,t}=0$, the same inclusion holds by construction. The result follows from \eqref{eq:joint_boundary}.
\end{proof}

\subsection{Exact rectangular projection}
\label{app:projectionproof}

\begin{proof}[Proof of Proposition~\ref{prop:projection}]
We prove both inclusions.

First take $m\in\widehat\M_t^{\proj}$. There exists $\gamma\in\Cset_t$ such that $m\in\M^\star(\gamma)$. Because $m$ is feasible under $\gamma$, for every $r$ we have $\gamma_{mr}\leq0$ and $L_{mr,t}\leq\gamma_{mr}$; hence $m\in\widehat\Vset_t^{\pos}$. Next let $j\in\widehat\Vset_t^{\cert}$. Then $U_{jr,t}\leq0$ for every $r$, so every $\gamma\in\Cset_t$ makes $j$ feasible. Optimality of $m$ under the selected $\gamma$ implies $Q_{m,t}\leq Q_{j,t}$. Therefore, $Q_{m,t}\leq\qcert_t$, proving membership in the right-hand side of \eqref{eq:operational}.

Conversely, suppose $m\in\widehat\Vset_t^{\pos}$ and $Q_{m,t}\leq\qcert_t$. We construct a matrix $\gamma\in\Cset_t$ under which $m$ is optimal. Since $m\in\widehat\Vset_t^{\pos}$, for every $r$ we have $L_{mr,t}\leq0$, so the interval $[L_{mr,t},U_{mr,t}]$ intersects $(-\infty,0]$. Choose a finite value $\gamma_{mr}\in[L_{mr,t},U_{mr,t}]\cap(-\infty,0]$ for every $r$. Then $\gamma_{mr}\leq0$ for every $r$, and therefore $m\in\Vset(\gamma)$.

Now consider any competitor $j$ with $Q_{j,t}<Q_{m,t}$. Such a competitor cannot belong to $\widehat\Vset_t^{\cert}$, because otherwise the definition of $\qcert_t$ would give $\qcert_t\leq Q_{j,t}<Q_{m,t}$, contradicting $Q_{m,t}\leq\qcert_t$. Hence $j\notin\widehat\Vset_t^{\cert}$, so there exists at least one coordinate $r(j)$ such that $U_{j,r(j),t}>0$. The interval $[L_{j,r(j),t},U_{j,r(j),t}]$ therefore contains a finite positive value. Choose
\[
\gamma_{j,r(j)}\in[L_{j,r(j),t},U_{j,r(j),t}]\cap(0,\infty).
\]
For every remaining coordinate $r\neq r(j)$ of this competitor, choose any finite $\gamma_{jr}\in[L_{jr,t},U_{jr,t}]$. Then $\gamma_{j,r(j)}>0$, so $j\notin\Vset(\gamma)$.

For every competitor $j$ with $Q_{j,t}\geq Q_{m,t}$, choose an arbitrary finite value $\gamma_{jr}\in[L_{jr,t},U_{jr,t}]$ for every $r$. Because $\Cset_t$ is a Cartesian product of the coordinate intervals, all of these choices can be made independently and together define a matrix $\gamma\in\Cset_t$.

Under this matrix, $m$ is feasible, and every method with a strictly smaller cost than $Q_{m,t}$ is infeasible. Therefore every feasible competitor $j$ satisfies $Q_{j,t}\geq Q_{m,t}$, so
\[
m\in\arg\min_{j\in\Vset(\gamma)}Q_{j,t}=\M^\star(\gamma).
\]
Hence $m\in\widehat\M_t^{\proj}$.
\end{proof}

The proof relies on coordinatewise freedom inside a rectangle. With a nonrectangular confidence region, setting one competitor infeasible can restrict the feasible values for other methods, and the operational formula need not remain exact.

\subsection{Main time-uniform theorem}
\label{app:mainproof}

\begin{proof}[Proof of Theorem~\ref{thm:main}]
Define the simultaneous confidence event
\[
\mathcal E:=\left\{\forall t\geq1,\ \forall m,r:\ L_{mr,t}\leq\Gamma_{mr,t}\leq U_{mr,t}\right\}.
\]
By Lemma~\ref{lem:gamma_interval}, $\Pp(\mathcal E)\geq1-\delta$.

Suppose that $\mathcal E$ holds and fix any time $t$. If $m\in\widehat\Vset_t^{\cert}$, then $U_{mr,t}\leq0$ for every $r$. Since $\Gamma_{mr,t}\leq U_{mr,t}$ on $\mathcal E$, we have $\Gamma_{mr,t}\leq0$ for every $r$, and hence $m\in\Vset_t^\star$. Therefore
\[
\widehat\Vset_t^{\cert}\subseteq\Vset_t^\star.
\]
Conversely, if $m\in\Vset_t^\star$, then $\Gamma_{mr,t}\leq0$ for every $r$. Since $L_{mr,t}\leq\Gamma_{mr,t}$ on $\mathcal E$, we have $L_{mr,t}\leq0$ for every $r$, so $m\in\widehat\Vset_t^{\pos}$. Thus
\[
\Vset_t^\star\subseteq\widehat\Vset_t^{\pos}.
\]

Now take any $m^\star\in\M_t^\star$. The preceding inclusion gives $m^\star\in\widehat\Vset_t^{\pos}$. Moreover, every $j\in\widehat\Vset_t^{\cert}$ is truly feasible, so the optimality of $m^\star$ over $\Vset_t^\star$ gives
\[
Q_{m^\star,t}\leq Q_{j,t},\qquad \forall j\in\widehat\Vset_t^{\cert}.
\]
Hence $Q_{m^\star,t}\leq q_t^{\cert}$, where $q_t^{\cert}=+\infty$ if $\widehat\Vset_t^{\cert}=\varnothing$. By the definition of $\widehat\M_t^{\proj}$, we have $m^\star\in\widehat\M_t^{\proj}$. Therefore
\[
\M_t^\star\subseteq\widehat\M_t^{\proj}.
\]

Since $t$ was arbitrary, all three inclusions hold for every $t$ on $\mathcal E$. Because $\Pp(\mathcal E)\geq1-\delta$, the result follows.
\end{proof}

\begin{corollary}\label{cor:stopping}
Let $T$ be any stopping time with respect to $(\G_t)$ that is finite almost surely. Under the conditions of Theorem~\ref{thm:main},
\begin{equation}
\Pp\!\left(\M_T^\star\subseteq\widehat\M_T^{\proj}\right)\geq1-\delta.
\end{equation}
For a possibly infinite stopping time, the corresponding statement is
$\Pp(T<\infty,\,\M_T^\star\nsubseteq\widehat\M_T^{\proj})\leq\delta$.
\end{corollary}

\begin{proof}[Proof of Corollary~\ref{cor:stopping}]
Let $\mathcal E$ denote the time-uniform event in Theorem~\ref{thm:main}. On $\mathcal E$, the inclusion $\M_t^\star\subseteq\widehat\M_t^{\proj}$ holds for every finite deterministic index $t$. Hence, on $\mathcal E\cap\{T<\infty\}$, it also holds at the realized index $T(\omega)$. Therefore
\[
\left\{T<\infty,\ \M_T^\star\not\subseteq\widehat\M_T^{\proj}\right\}\subseteq\mathcal E^c,
\]
and consequently
\[
\Pp\left(T<\infty,\ \M_T^\star\not\subseteq\widehat\M_T^{\proj}\right)\leq\Pp(\mathcal E^c)\leq\delta.
\]
If $T<\infty$ almost surely, this reduces to
\[
\Pp\left(\M_T^\star\subseteq\widehat\M_T^{\proj}\right)\geq1-\delta.
\]
\end{proof}

\subsection{Contraction rate}
\label{app:contraction}

\begin{proposition}\label{prop:contraction}
Suppose that $a_{rt}\in[0,1]$ and the exposure is persistent almost surely, in the sense that $A_{r,t}\asymp t$. For fixed $\rho>0$ and $\eta_{mr}\in(0,1)$, the empirical-Bernstein confidence sequence based on the gamma-exponential mixture boundary satisfies
\[
\operatorname{rad}_{mr,t}=O\left(\sqrt{\frac{\log t+\log(1/\eta_{mr})}{t}}\right)\to0.
\]
If $\eta_{mr}\asymp\delta/(MR)$, then
\[
\operatorname{rad}_{mr,t}=O\left(\sqrt{\frac{\log t+\log(MR/\delta)}{t}}\right)\to0.
\]
\end{proposition}

\begin{proof}[Proof of Proposition~\ref{prop:contraction}]
Recall that
\[
\operatorname{rad}_{mr,t}=\frac{b_{\eta_{mr},\rho}(V_{mr,t})}{A_{r,t}},
\qquad
V_{mr,t}=\sum_{s=1}^t(X_{mrs}-\widehat X_{mrs})^2.
\]
Since
\[
X_{mrs}-\widehat X_{mrs}=a_{rs}(B_{mrs}-\widehat B_{mrs}),
\]
with $B_{mrs}\in\{0,1\}$, $\widehat B_{mrs}\in[0,1]$, and $a_{rs}\in[0,1]$, we have
\[
V_{mr,t}
=\sum_{s=1}^t a_{rs}^2(B_{mrs}-\widehat B_{mrs})^2
\leq\sum_{s=1}^t a_{rs}^2
\leq\sum_{s=1}^t a_{rs}
=A_{r,t}.
\]
Because $A_{r,t}\asymp t$, there exist constants $c_A,C_A>0$ and $t_0$ such that, for all $t\geq t_0$,
\[
c_A t\leq A_{r,t}\leq C_A t.
\]
Hence,
\[
0\leq V_{mr,t}\leq C_A t.
\]

Since $V_{mr,t}$ is nondecreasing, it either remains bounded or tends to infinity. For the gamma-exponential mixture boundary, Proposition~2 of \citet{howard2021time} gives
\[
b_{\eta_{mr},\rho}(v)=O\left(\sqrt{v\left[\log v+\log(1/\eta_{mr})\right]}\right)
\]
as $v\to\infty$, for fixed $\rho>0$ and $\eta_{mr}\in(0,1)$. Therefore, if $V_{mr,t}\to\infty$,
\[
b_{\eta_{mr},\rho}(V_{mr,t})
=O\left(\sqrt{t\left[\log t+\log(1/\eta_{mr})\right]}\right),
\]
where we have used $V_{mr,t}\leq C_A t$. Since $A_{r,t}\geq c_A t$,
\[
\operatorname{rad}_{mr,t}=\frac{b_{\eta_{mr},\rho}(V_{mr,t})}{A_{r,t}}=O\left(\sqrt{\frac{\log t+\log(1/\eta_{mr})}{t}}\right)
\to0.
\]

If instead $V_{mr,t}$ remains bounded, then $b_{\eta_{mr},\rho}(V_{mr,t})=O(1)$ for fixed $\eta_{mr}$ and $\rho$, while $A_{r,t}\asymp t$. Hence,
\[
\operatorname{rad}_{mr,t}=O(t^{-1})\to0.
\]
Thus, in either case,
\[
\operatorname{rad}_{mr,t}=O\left(\sqrt{\frac{\log t+\log(1/\eta_{mr})}{t}}\right)\to0.
\]

If $\eta_{mr}\asymp\delta/(MR)$, then
\[
\log(1/\eta_{mr})=\log(MR/\delta)+O(1),
\]
and therefore
\[
\operatorname{rad}_{mr,t}=O\left(\sqrt{\frac{\log t+\log(MR/\delta)}{t}}\right)\to0.
\]

On the joint confidence event, the confidence bounds satisfy
\[
U_{mr,t}\le\Gamma_{mr,t}+2\mathrm{rad}_{mr,t},
\qquad
L_{mr,t}\ge\Gamma_{mr,t}-2\mathrm{rad}_{mr,t}.
\]
If $\Gamma_{mr,t}>2\mathrm{rad}_{mr,t}$ for some $r$, then $L_{mr,t}>0$, so $m\notin\widehat{\mathcal V}^{\mathrm{pos}}_t$. Alternatively, suppose that a competitor $j$ satisfies $Q_{j,t}<Q_{m,t}$ and $\Gamma_{jr,t}+2\mathrm{rad}_{jr,t}\le0$ for every $r$. Then $U_{jr,t}\le0$ for every $r$, so
$j\in\widehat{\mathcal V}^{\mathrm{cert}}_t$ and
\[
q_t^{\mathrm{cert}}\le Q_{j,t}<Q_{m,t}.
\]
The projection rule therefore excludes $m$ in either case. Since the same event guarantees $\mathcal M^\star_t\subseteq\widehat{\mathcal M}^{\mathrm{proj}}_t$, excluding every non-oracle method gives exact recovery.

Finally, $M,R<\infty$ and persistent exposure imply $\max_{m,r}\mathrm{rad}_{mr,t}\to0$. Under the uniform margins stated in Section~\ref{sec:oracle-contraction}, every non-oracle method therefore satisfies an exclusion condition for all sufficiently large $t$.
\end{proof}

\subsection{Certification and oracle identification}
\label{app:certification-limits}

\begin{proof}[Proof of Proposition~\ref{prop:certification-limits}]
Let $\mathcal F_n:=\sigma(B_1,\ldots,B_n)$ and $E_n:=\{T_{\mathrm{cert}}\le n\}$. Since $E_n\in\mathcal F_n$, its probability under the Bernoulli product law is a finite sum of terms of the form $p^k(1-p)^{n-k}$, and is therefore continuous in $p$. The assumption gives $\mathbb P_p(E_n)\le\eta$ for every $p>\tau$. Taking $p\downarrow\tau$ yields $\mathbb P_\tau(E_n)\le\eta$. Finally, since $E_n\uparrow\{T_{\mathrm{cert}}<\infty\}$,
\[
\mathbb P_\tau(T_{\mathrm{cert}}<\infty)=\lim_{n\to\infty}\mathbb P_\tau(E_n)\le\eta.
\]
\end{proof}

\paragraph{Implication for oracle identification.}
In the two-candidate construction of Section~\ref{sec:coverage-boundary}, candidate~2 is the unique oracle whenever $p>\tau$. Let $(\widehat{\mathcal S}_t)_{t\ge1}$ be an $(\mathcal F_t)$-adapted rule satisfying (\ref{eq:ccsmcsdef}) for every $p\in(0,1)$, and define
\[
T_2:=\inf\{t\ge1:2\notin\widehat{\mathcal S}_t\}.
\]
For $p>\tau$, excluding candidate~2 violates oracle coverage, so $\sup_{p>\tau}\mathbb P_p(T_2<\infty)\le\delta$. Proposition~\ref{prop:certification-limits} therefore gives $\mathbb P_\tau(T_2<\infty)\le\delta$. Since reporting $\{1\}$ implies excluding candidate~2,
\[
\mathbb P_\tau\!\left(\exists t\ge1:\widehat{\mathcal S}_t=\{1\}\right)\le\mathbb P_\tau(T_2<\infty)\le\delta.
\]

\paragraph{Finite-sample lower bound near the boundary.}
Now let $p=\tau-\Delta$ with $\Delta\in(0,\tau)$, so that candidate~1 is strictly feasible and uniquely optimal. The difficulty of distinguishing this risk from the boundary risk is measured by the Bernoulli relative entropy
\[
\operatorname{kl}(a,b):=a\log\frac{a}{b} +(1-a)\log\frac{1-a}{1-b},\qquad a,b\in(0,1).
\]
For any $\beta\in(0,1-\delta)$ and integer $n\ge1$,
\[
\mathbb P_{\tau-\Delta}(T_2\le n)\ge1-\beta\quad\Longrightarrow\quad n\ge\frac{\operatorname{kl}(1-\beta,\delta)}{\operatorname{kl}(\tau-\Delta,\tau)}.
\]
Indeed, let $E_n:=\{T_2\le n\}$.
The boundary result gives $\mathbb P_\tau(E_n)\le\delta$.
Applying the log-sum inequality to $E_n$ and its complement yields
\[
n\operatorname{kl}(\tau-\Delta,\tau)\ge\operatorname{kl}\!\left(\mathbb P_{\tau-\Delta}(E_n),\mathbb P_\tau(E_n)\right)\ge\operatorname{kl}(1-\beta,\delta),
\]
where the last inequality uses
$\mathbb P_{\tau-\Delta}(E_n)\ge1-\beta>\delta\ge\mathbb P_\tau(E_n)$.
This proves the bound.

For fixed $\tau\in(0,1)$,
\[
\operatorname{kl}(\tau-\Delta,\tau)=\frac{\Delta^2}{2\tau(1-\tau)}+o(\Delta^2)
\qquad\text{as }\Delta\downarrow0.
\]
Hence the lower bound is of order $\Delta^{-2}$ for fixed $\delta$ and $\beta$. Reporting the correct singleton $\{1\}$ by time $n$ implies $T_2\le n$, so the same lower bound applies to high-probability singleton identification.

\section{The Weak Sequential Model Confidence Set}
\label{sub:weak_smcs}

The Weak Sequential Model Confidence Set (Weak-SMCS) of \citet{arnold2026sequential} is a score-based procedure for sequentially identifying weakly superior forecasting methods. Let $\ell_{m,t}$ denote the loss of method $m$ at time $t$, and define the pairwise loss difference between methods $m$ and $j$ as
\[
d_{mj,t}=\ell_{m,t}-\ell_{j,t}.
\]
The average conditional expected loss difference up to time $t$ is
\[
\Delta_{mj,t}=\frac{1}{t}\sum_{s=1}^{t}\mathbb{E}\left[d_{mj,s}\mid\mathcal G_{s-1}\right].
\]
The set of weakly superior methods at time $t$ is defined as
\[
\mathcal M_t^{w,\star}=\left\{m:\Delta_{mj,t}\leq0,\ \forall j\neq m\right\}.
\]
Unlike uniformly weak superiority, this target is allowed to vary over time, so a method excluded at an earlier stage may become weakly superior again later.

In our implementation, the loss is the Winkler score. For a prediction interval $[\underline{Y}_{m,t},\overline{Y}_{m,t}]$ with nominal miscoverage level $\alpha$, the Winkler score is
\[
\ell_{m,t}=(\overline{Y}_{m,t}-\underline{Y}_{m,t})+\frac{2}{\alpha}(\underline{Y}_{m,t}-Y_t)\mathbf{1}\{Y_t<\underline{Y}_{m,t}\}+\frac{2}{\alpha}(Y_t-\overline{Y}_{m,t})\mathbf{1}\{Y_t>\overline{Y}_{m,t}\}.
\]
Although the Winkler score itself is unbounded as a function of $Y_t$, the difference between the scores of two fixed prediction intervals is conditionally bounded. Following Section~3.4 of \citet{arnold2026sequential}, for each ordered pair $(m,j)$ and time $t$, we compute the predictable bound
\[
b_{mj,t}=\sup_{y\in\mathbb{R}}\left|WS_{m,t}(y)-WS_{j,t}(y)\right|.
\]
Since $b_{mj,t}$ depends only on the prediction intervals issued before observing $Y_t$, it is predictable. We then standardize the pairwise loss difference as
\[
\widetilde d_{mj,t}=\frac{d_{mj,t}}{b_{mj,t}},
\]
with $\widetilde d_{mj,t}=0$ when $b_{mj,t}=0$. By construction,
\[
|\widetilde d_{mj,t}|\leq1.
\]

This predictable normalization converts the conditionally bounded loss differences into uniformly bounded ones. Accordingly, the Weak-SMCS used in our experiments targets weak superiority with respect to the scaled loss differences. Define
\[
\widetilde\Delta_{mj,t}=\frac{1}{t}\sum_{s=1}^{t}\mathbb{E}\left[\widetilde d_{mj,s}\mid\mathcal G_{s-1}\right]=\frac{1}{t}\sum_{s=1}^{t}\frac{\mathbb{E}[d_{mj,s}\mid\mathcal G_{s-1}]}{b_{mj,s}}.
\]
The corresponding weakly superior set is
\[
\widetilde{\mathcal M}_t^{w,\star}=\left\{m:\widetilde\Delta_{mj,t}\leq0,\ \forall j\neq m\right\}.
\]

For the standardized differences, we use the bounded-difference construction of Proposition~3.6 in \citet{arnold2026sequential}. Since $|\widetilde d_{mj,t}|\leq1$, we set $c=2$ and $\lambda=1/4$. The predictable centering sequence is given by the running mean of previous standardized differences,
\[
\gamma_{mj,1}=0,\qquad \gamma_{mj,t}=\frac{1}{t-1}\sum_{s=1}^{t-1}\widetilde d_{mj,s},\quad t\geq2,
\]
with $\gamma_{mj,t}$ restricted to $[-1,1]$. Define
\[
V_{mj,t}=\sum_{s=1}^{t}\left(\widetilde d_{mj,s}-\gamma_{mj,s}\right)^2,
\]
and
\[
\psi_{E,c}(\lambda)=\frac{-\log(1-c\lambda)-c\lambda}{c^2}.
\]
For each ordered pair $(m,j)$ and candidate value $x\in\mathbb R$, define
\[
M_{mj,t}(x)=\exp\left\{\lambda\sum_{s=1}^{t}\widetilde d_{mj,s}-\lambda tx-\psi_{E,c}(\lambda)V_{mj,t}\right\}.
\]
In particular, evaluating the process at the boundary value $x=0$ gives
\[
M_{mj,t}(0)=\exp\left\{\lambda\sum_{s=1}^{t}\widetilde d_{mj,s}-\psi_{E,c}(\lambda)V_{mj,t}\right\}.
\]
Following the joint confidence-region construction of \citet{arnold2026sequential}, let $X=(x_{mj})_{m\neq j}$ denote a candidate matrix of pairwise conditional-average loss differences and define
\[
M_t(X)=\frac{1}{M(M-1)}\sum_{m\neq j}M_{mj,t}(x_{mj}),\qquad \mathcal C_t=\left\{X:M_t(X)\leq\frac{1}{\delta}\right\}.
\]
The set $\mathcal C_t$ is a joint confidence region for the matrix of scaled pairwise conditional-average loss differences. The dynamic Weak-SMCS is obtained by checking, for each method $m$ and every competitor $j\neq m$, whether the joint confidence region still contains a candidate matrix under which method $m$ is not worse than method $j$:
\[
\widehat{\mathcal M}^{w}_t=\left\{m:\mathcal C_t\cap\{X:x_{mj}\leq0\}\neq\varnothing,\ \forall j\neq m\right\}.
\]

Since $M_{mj,t}(x)$ is nonnegative, convex, and decreasing in $x$, the monotonicity argument of \citet{arnold2026sequential} reduces the intersection condition to the boundary value $x_{mj}=0$. For fixed $(m,j)$, the remaining coordinates of $X$ can be taken arbitrarily large, making their contributions to $M_t(X)$ arbitrarily close to zero. Hence,
\[
\inf_{X:\,x_{mj}\leq0}M_t(X)=\frac{M_{mj,t}(0)}{M(M-1)}.
\]

\[
\mathcal C_t\cap\{X:x_{mj}\leq0\}\neq\varnothing\quad\Longleftrightarrow\quad M_{mj,t}(0)<\frac{M(M-1)}{\delta}.
\]
Equivalently, method $m$ is excluded at time $t$ if there exists a competitor $j\neq m$ such that
\[
M_{mj,t}(0)\geq\frac{M(M-1)}{\delta}.
\]
Thus, the estimated dynamic Weak-SMCS can be written as
\[
\widehat{\mathcal M}^{w}_t=\left\{m:M_{mj,t}(0)<\frac{M(M-1)}{\delta},\ \forall j\neq m\right\}.
\]
The resulting sequence satisfies the time-uniform guarantee
\[
\Pr\left(\widetilde{\mathcal M}^{w,\star}_t\subseteq\widehat{\mathcal M}^{w}_t,\ \forall t\right)\geq1-\delta.
\]

Because the Weak-SMCS is defined pointwise at each time $t$, previously excluded methods are allowed to re-enter when the accumulated evidence changes. We use this dynamic version rather than its running-intersection counterpart.

We include Weak-SMCS as a score-based benchmark in simulations. Its inferential target differs from that of CC-SMCS: Weak-SMCS compares methods through predictably standardized Winkler-score differences, whereas CC-SMCS treats coverage as an explicit constraint and interval width as the efficiency criterion. Therefore, the two procedures target different superior sets and may exhibit different model-set paths even when applied to the same sequence of prediction intervals.

\section{The setting of $\rho$}
\label{sec:select_rho}

The parameter $\rho$ determines the intrinsic-time region over which the mixture boundary in equation~(\ref{eq:boundary}) is relatively tight. Following \citet{howard2021time}, we also use the tuning rule derived for the normal mixture boundary to guide the choice of $\rho$ for the gamma-exponential mixture boundary. 

In our construction, the total two-sided crossing probability assigned to a confidence sequence is $\eta$, so each one-sided boundary is assigned crossing probability $\eta/2$. Following the normal-mixture tuning rule of \citet{howard2021time}, for a prespecified target intrinsic time $w>0$, we set
\[
\rho(w)=\frac{w}{-W_{-1}\left(-\dfrac{\eta^2}{e l_0^2}\right)-1}.
\]
Here, $w>0$ denotes the target intrinsic time, $\rho>0$ is the mixture tuning parameter, $\eta\in(0,1)$ is the corresponding two-sided crossing probability, $l_0\geq1$ is the initial supermartingale constant, and $W_{-1}(\cdot)$ denotes the lower branch of the Lambert $W$ function satisfying $W(x)e^{W(x)}=x$. In our scalar setting, we take $l_0=1$.

To determine an appropriate target intrinsic time $w$, recall that
\[
A_{r,t}=\sum_{s=1}^{t}a_{rs}
\]
denotes the cumulative exposure. For method $m$ and constraint $r$, the empirical-Bernstein intrinsic time is
\[
V_{mr,t}=\sum_{s=1}^{t}\left(X_{mrs}-\widehat{X}_{mrs}\right)^2,
\]
and the corresponding one-sided confidence radius is
\[
R_{mr,t}=\frac{b_{\eta,\rho}(V_{mr,t})}{A_{r,t}}.
\]
Because $V_{mr,t}$ depends on the realized observations, it is unknown before the forecasting period begins and therefore cannot be used directly to prespecify $\rho$.

To address this issue, we follow \citet{howard2021time}, who use a Hoeffding variance process to provide a deterministic intrinsic-time scale for bounded observations, and adopt the same idea in our weighted setting. Since $B_{mrs}\in\{0,1\}$,
\[
B_{mrs}-\tau_r\in[-\tau_r,\,1-\tau_r].
\]
Hence, the weighted increment $a_{rs}(B_{mrs}-\tau_r)$ lies in
\[
a_{rs}(B_{mrs}-\tau_r)\in[-a_{rs}\tau_r,\,a_{rs}(1-\tau_r)],
\]
whose range length is
\[
a_{rs}(1-\tau_r)-(-a_{rs}\tau_r)=a_{rs}.
\]
By Hoeffding's lemma, the corresponding centered increment is conditionally sub-Gaussian with variance upper bound
\[
\frac{a_{rs}^2}{4}.
\]
Following the Hoeffding variance construction of \cite{howard2021time}, we therefore use
\[
V^H_{r,t}=\frac{1}{4}\sum_{s=1}^{t}a_{rs}^2
\]
as the reference intrinsic time for tuning $\rho$. Under the unit-exposure setting used in our experiments, this reference path is deterministic and can be computed before the forecasting period begins.

Let $T$ denote the forecasting horizon. We prespecify three target forecasting times corresponding to the $25\%$, $50\%$, and $75\%$ points of the forecasting horizon,
\[
t_q^\star=qT,\qquad q\in\{0.25,0.50,0.75\}.
\]
Under the unit-exposure setting adopted in this paper,
\[
a_{rs}=1,\quad A_{r,t}=t,\quad V_{r,t}^{\mathrm{H}}=\frac{t}{4}.
\]

The corresponding target intrinsic times are therefore
\[
w_{0.25}=\frac{T}{16},\qquad w_{0.50}=\frac{T}{8},\qquad w_{0.75}=\frac{3T}{16}.
\]
These target intrinsic times generate three candidate tuning parameters,
\[
\rho_q=\rho(w_q),\qquad q\in\{0.25,0.50,0.75\}.
\]

To select among these candidates, we evaluate the corresponding normal-mixture reference radius along the deterministic Hoeffding path
\[
V_t^{\mathrm{H}}=\frac{t}{4}.
\]
Let $u_{\eta,\rho}(v)$ denote the two-sided normal mixture boundary used for tuning. For each candidate $\rho_q$, we compute
\[
u_{\eta,\rho}\!\left(V_{r,t}^{\mathrm H}\right)=\sqrt{\left(V_{r,t}^{\mathrm H}+\rho\right)\log\left(\frac{V_{r,t}^{\mathrm H}+\rho}{\eta^2\rho}\right)}, \quad\overline{R}^{\mathrm{ref}}(\rho_q)=\frac{1}{T}\sum_{t=1}^{T}\frac{u_{\eta,\rho_q}(t/4)}{t},
\]
and select
\[
\rho^\star=\arg\min_{\rho_q:\,q\in\{0.25,0.50,0.75\}}\overline{R}^{\mathrm{ref}}(\rho_q).
\]
The selected $\rho^\star$ is fixed before the forecasting outcomes are observed.

The Hoeffding reference intrinsic time is used only for the prespecified choice of $\rho$. Once $\rho^\star$ has been selected, the actual confidence sequence is constructed using the realized empirical-Bernstein intrinsic time $V_{mr,t}$. Under the unit-exposure setting, the resulting one-sided confidence radius is
\[
R_{mr,t}=\frac{b_{\eta,\rho^\star}(V_{mr,t})}{t}.
\]

Thus, for a fixed error allocation $\eta$, the prespecified choice of $\rho^\star$ depends only on the forecasting horizon $T$. Since the three simulation settings share the same forecasting horizon, they have the same candidate values of $\rho$ and the same selected $\rho^\star$. We therefore report a single $\rho^\star$ for the three simulations. Table~\ref{tab:choe_rho} summarizes the selected values of $\rho^\star$ for the simulation setting and the four empirical datasets.

\begin{table}[htbp]
\centering
\caption{Selected values of $\rho$}
\label{tab:choe_rho}
\resizebox{\linewidth}{!}{
\begin{tabular}{lcccccc}
\toprule
The selection of $\rho$& $\delta$ &Simulation Data Set & Data Set 1  & Data Set 2&  Data Set 3 &  Data Set 4 \\
\midrule
\multirow{2}{*}{$\rho^{\star}$}
& 0.10 &  17.616&20.494& 29.584 & 30.946 & 8.151 \\
& 0.25 &21.752& 25.305 & 36.529 &38.211& 10.064 \\
\bottomrule
\end{tabular}}
\footnotesize{\textit{Note:} Data Sets 1, 2, 3, 4 respectively represent the data of French electricity price, electricity load, weather, and wind power.}
\end{table}

\section{More details of Simulations}
\label{sec:detail_sim}
\subsection{The computation of coverage and width for each method}
In simulations, we design six methods. Specifically, the  coverage rates for M2-M6 are computed by:
\[
\begin{aligned}
P\left\{Y_t\in C_{mr,t}\mid\mathcal{G}_{t-1}\right\}&=P\left\{\widehat{Y}_t-\sigma_{mr,t}z\le Y_t\le\widehat{Y}_t+\sigma_{mr,t}z\mid\mathcal{G}_{t-1}\right\}\\
&=P\left\{-\sigma_{mr,t}z\le Y_t-\widehat{Y}_t\le\sigma_{mr,t}z\mid\mathcal{G}_{t-1}\right\}\\
&=P\left\{-\frac{\sigma_{mr,t}z}{\sigma_t}\le \frac{Y_t-\widehat{Y}_t}{\sigma_t}\le \frac{\sigma_{mr,t}z}{\sigma_t}\,\middle|\,\mathcal{G}_{t-1}\right\}\\
&=\Phi(\frac{\sigma_{mr,t}z}{\sigma_t})-\Phi(-\frac{\sigma_{mr,t}z}{\sigma_t})\\&=2\Phi(\frac{\sigma_{mr,t}z}{\sigma_t})-1\\
&=2\Phi(\exp\{(\theta_{mr}-l_{t})|h_{t}|\}z)-1
\end{aligned}
\]
The last equality follows from $\sigma_{mr,t}=\sigma_t^0\exp\{\theta_{mr}|h_t|\}$. Hence, the conditional miscoverage probability is:
\[
\Pr\left\{Y_t\notin C_{mr,t}\mid\mathcal{G}_{t-1}\right\}
=2\left[1-\Phi\left(\exp\{(\theta_{mr}-l_{t})|h_t|\}z\right)\right].
\]
This conditional miscoverage probability can be related to our framework as follows:
\[
\Gamma_{mr,t}=
\begin{cases}
A_{r,t}^{-1}\sum_{s=1}^t\mu_{mrs}, & A_{r,t}>0,\\
0, & A_{r,t}=0.
\end{cases}
\]
where $\mu_{mrt}=\mathbb{E}[X_{mrt}\mid\mathcal{G}_{t-1}]=a_{rt}\left[\mathbb{E}[B_{mrt}\mid\mathcal{G}_{t-1}]-\tau_r\right]$ and $B_{mrt}=\mathbf{1}\{Y_t\notin C_{mr,t}\}$. Hence, when $A_{r,t}>0$,
\[
\begin{aligned}
\Gamma_{mr,t}
&=A_{r,t}^{-1}\sum_{s=1}^t a_{rs}\left[ \mathbb{E}[B_{mrs}\mid\mathcal{G}_{s-1}]-\tau_r \right]\\
&=A_{r,t}^{-1}\sum_{s=1}^t
a_{rs}\left[\Pr\left\{Y_s\notin C_{mr,s}\mid\mathcal{G}_{s-1}\right\}-\tau_r\right]\\
&=A_{r,t}^{-1}\sum_{s=1}^t a_{rs}\left[2\left\{1-\Phi\left(\exp\{(\theta_{mr}-l_{s})|h_s|\}z\right)\right\}-\tau_r\right].
\end{aligned}
\]
Similarly, for M1, whose prediction interval is $C_{1r,t}=[\widehat{Y}_t+\kappa\sigma_{4r,t}-z\sigma_{4r,t},\,\widehat{Y}_t+\kappa\sigma_{4r,t}+z\sigma_{4r,t}]$, the conditional miscoverage probability is
\[
\Pr\{Y_t\notin C_{1r,t}\mid\mathcal G_{t-1}\}=1-\left[\Phi\left(\frac{(\kappa+z)\sigma_{4r,t}}{\sigma_t}\right)-\Phi\left(\frac{(\kappa-z)\sigma_{4r,t}}{\sigma_t}\right)\right].
\]
Therefore, $\Gamma_{1r,t}$ is given by:
\[
\Gamma_{1r,t}=A_{r,t}^{-1}\sum_{s=1}^t a_{rs}\left[1-\left[\Phi\left(\exp\{(\theta_{4r}-l_{s})|h_{s}|\}(\kappa+z)\right)\right]+\left[\Phi\left(\exp\{(\theta_{4r}-l_{s})|h_{s}|\}(\kappa-z)\right)\right]-\tau_r\right]
\]

In the stationary case, the relationship between $\theta_{mr}$ and $l$ directly determines the feasibility of M2--M6. With constant $l_t=l$ and $\tau=\alpha$, if $\theta_{mr}<l$, then $\sigma_{mr,t}<\sigma_t$, the conditional miscoverage probability exceeds $\alpha$, and hence $\Gamma_{mr,t}>0$; if $\theta_{mr}=l$, then $\Gamma_{mr,t}=0$; and if $\theta_{mr}>l$, then $\sigma_{mr,t}>\sigma_t$, the conditional miscoverage probability is below $\alpha$, and hence $\Gamma_{mr,t}<0$. In our construction, M4 corresponds to the case $\theta_{4r}=l$ and therefore attains the nominal coverage level in the stationary setting. M1 is constructed using M4 as a reference: it has the same interval width as M4, but its center is shifted away from $\widehat{Y}_t$. This shift reduces the coverage probability without changing the interval width. The parameter $\kappa$ controls the magnitude of the shift: $\kappa=0$ recovers M4, while any nonzero $\kappa$ lowers the coverage relative to M4. We set $\kappa=0.5$, which yields $\Gamma_{1r,t}>0$ and therefore makes M1 infeasible under the coverage constraint while preserving the same width as M4.

In the dynamic settings, as in Simulations 2 and 3, $l_t$ changes over time. For M2--M6, the relation between $\theta_{mr}$ and $l_t$ determines only the instantaneous contribution to the cumulative excess miscoverage, while the sign of $\Gamma_{mr,t}$ also depends on the coverage history accumulated at previous time points. M1 is affected differently because its interval is shifted relative to the predictive center. Its instantaneous miscoverage therefore depends jointly on the scale ratio between M4 and the true forecast-error scale and on the fixed center-shift parameter $\kappa$. Consequently, the feasibility of M1 in the dynamic settings is determined from its own accumulated excess miscoverage $\Gamma_{1r,t}$ rather than from the relation between $\theta_{4r}$ and $l_t$ alone. According to Definition~\ref{def:oracle}, we determine the oracle from the cumulative coverage risk together with the cumulative interval width. For M2--M6, the interval width $2z\sigma_{mr,t}$ increases monotonically with $\sigma_{mr,t}$, while M1 has the same width as M4 by construction. The constrained oracle is therefore the method with the smallest cumulative interval width among those satisfying the cumulative coverage constraint.

\subsection{Parameter Settings}

We specify the candidate interval scale as $\sigma_{mr,t}=\sigma_t^0\exp\{\theta_{mr}|h_t|\}$, where $\sigma_t^0=\exp\{H_{t-1}\}$ and $\phi=0.5$. Given $\sigma_{mr,t}$, the corresponding coverage probability and interval width can be computed directly. Since $h_t$ is fixed by the data-generating process, it remains only to specify $\theta_{mr}$. For convenience, we introduce a common scaling parameter $\xi$ and, for $r=1$, set
\[
\theta_{mr}=\xi(-2,-1,0,1,2),\qquad m=2,3,4,5,6
\]
corresponding to M2--M6, respectively. We set $\xi=0.15$. For M1, we construct the interval as $[\widehat{Y}_{t}+\kappa\sigma_{4r,t}-z\sigma_{4r,t},\widehat{Y}_{t}+\kappa\sigma_{4r,t}+z\sigma_{4r,t}]$, where $\kappa$ controls the magnitude of the center shift. We set $\kappa=0.5$. Thus, M1 has the same interval width as M4 but a deliberately shifted center, which induces undercoverage. For all methods, we set $z=\Phi^{-1}(1-\alpha/2)$ with $\alpha=0.05$.

For the CC-SMCS parameters used in the simulations, we set $R=1$, $a_{rt}=1$, and $\epsilon_r=0$, so that $\tau_r=\alpha_r=0.05$. Since there are six candidate methods and one monitored constraint, we use a balanced error allocation with $\pi_{mr}=1/6$, giving $\eta_{mr}=\delta/6$, where $\delta\in\{0.10,0.25\}$. We set $w_t=1$, so that the efficiency objective is the cumulative mean interval width. The mixture tuning parameter $\rho^\star$ is fixed before evaluation according to the procedure described in Appendix~\ref{sec:select_rho}. The selected values are $\rho^\star=17.616$ for $\delta=0.10$ and $\rho^\star=21.752$ for $\delta=0.25$, as reported in Table~\ref{tab:choe_rho}.

\subsection{Connection Between the Parameters and the Three Simulations}
\label{sub:para_sim}
Across the three simulations, the candidate-specific parameters $\theta_{mr}$ remain unchanged. Instead, we modify the true forecast-error scale to generate stationary, structural-break, and gradual-drift environments.
We set the master random seed to 42 and generate independent replication-specific seeds using NumPy's SeedSequence. The same replication-specific seeds are reused across the three simulation designs to implement paired common random numbers.
\paragraph{Simulation 1}
The stationary setting assumes that the true scale $\sigma_t$ follows its baseline process, $\sigma_t^0=\exp\{H_{t-1}\}$, throughout the forecasting horizon. Under this setting, M4 is the constrained oracle at all time points because $\theta_{4r}=0$. Since $\theta_{mr}$ determines the conditional miscoverage probability through $\sigma_{mr,t}$ and $R=1$, we have
\[
\Gamma_{1r,t}>0,\qquad
\Gamma_{2r,t}>0,\qquad
\Gamma_{3r,t}>0,\qquad
\Gamma_{4r,t}=0,\qquad
\Gamma_{5r,t}<0,\qquad
\Gamma_{6r,t}<0.
\]
Hence, M1--M3 violate the coverage constraint, M4 lies exactly on the coverage boundary, and M5--M6 are conservative. Under the width specification used in the selection step,
\[
Q_{2r,t}<Q_{3r,t}<Q_{1r,t}=Q_{4r,t}<Q_{5r,t}<Q_{6r,t},
\qquad \forall t\in{T}.
\]
Therefore, M4 is the narrowest method satisfying the population coverage constraint and is the unique constrained oracle throughout Simulation 1.

\paragraph{Simulation 2}
This simulation introduces a permanent structural break in the true forecast-error scale while keeping all candidate procedures unchanged. Specifically, after $t=400$, the true scale changes from $\sigma^{0}_t$ to $\sigma^{0}_t\exp\{1.75\xi|h_t|\}$. The candidate interval constructions remain unchanged, whereas their conditional coverage probabilities shift after the break. The forecast-error scales then satisfy
\[
\sigma_{4r,t}=\sigma^{0}_t<\sigma_{5r,t}=\sigma^{0}_t\exp\{\xi|h_t|\}<\sigma_{t}=\sigma^{0}_t\exp\{1.75\xi|h_t|\}<\sigma_{6r,t}=\sigma^{0}_t\exp\{2\xi|h_t|\}.\label{eq:sig}
\]

M4 lies on the coverage boundary before the break and loses
feasibility at $t=401$. After the break, the oracle is M5 when
$\Gamma_{5r,t}\le0$ and M6 otherwise. The transition therefore
depends on the coverage margin accumulated by M5 before the break.


\paragraph{Simulation 3}
Instead of the abrupt break, this simulation sets a gradual linear drift in the true forecast-error scale. Specifically, the coefficient governing the scale distortion increases linearly from $0$ to $2\xi$ over the evaluation period, so that the true scale is given by $\sigma_t=\sigma_t^0\exp{2\xi\frac{t-1}{T-1}|h_t|}$. As the scale changes continuously over time, the conditional coverage probabilities of the candidate methods also evolve gradually. Since M4 is initially located on the coverage boundary, it loses feasibility shortly after the drift begins, and the oracle therefore moves quickly from M4 to M5. In theory, $\sigma_{5r,t}$ remains larger than $\sigma_t$ until approximately $t=1500$, while $\sigma_{6r,t}$ remains larger than $\sigma_t$ until $t=3000$. This setting therefore provides a more challenging environment for evaluating whether CC-SMCS can maintain validity and adapt to a smoothly evolving oracle.

\subsection{Additional Results}
\label{sub:addi_re}
Figure~\ref{fig:sim_025} presents the simulation results for $\delta=0.25$. In Simulation 1, where M4 is the fixed oracle, CC-SMCS retains M4 together with M5 in the final model set, consistent with the oracle-retention guarantee in Theorem~\ref{thm:main}. In Simulation 2, the constrained oracle changes from M4 to M5 and subsequently to M6. Figure~\ref{fig:sim2_025} shows that CC-SMCS adapts to these oracle changes while continuously retaining the current oracle in the model set. Simulation 3 exhibits a pattern similar to that obtained for $\delta=0.10$ with a linear drift. M4 becomes undercovered from $t\geq2$, after which M5 becomes the constrained oracle. Although $\sigma_{5r,t}$ remains larger than the true scale $\sigma_t$ until approximately $t=1500$, M5 does not immediately cease to be feasible once $\sigma_{5r,t}<\sigma_t$, because feasibility is determined by cumulative rather than instantaneous miscoverage. M5 therefore remains the oracle until its accumulated coverage margin is exhausted and $\Gamma_{5r,t}$ becomes positive, after which M6 becomes the constrained oracle. In Figure~\ref{fig:sim3_025}, M5 remains in the CC-SMCS throughout the evaluation period. The inclusion probability of M6 first decreases and then increases, and M4 is gradually excluded as the drift strengthens. Weak-SMCS, which targets the Winkler score rather than coverage-constrained interval width, exhibits different model-set dynamics. In Simulation 1, it retains the oracle throughout the evaluation period, whereas in Simulations 2 and 3 it tends to retain methods with undercoverage for longer after the oracle changes.
\begin{figure}[htbp]
\centering
\begin{subfigure}{\linewidth}
\centering
\includegraphics[width=\linewidth]{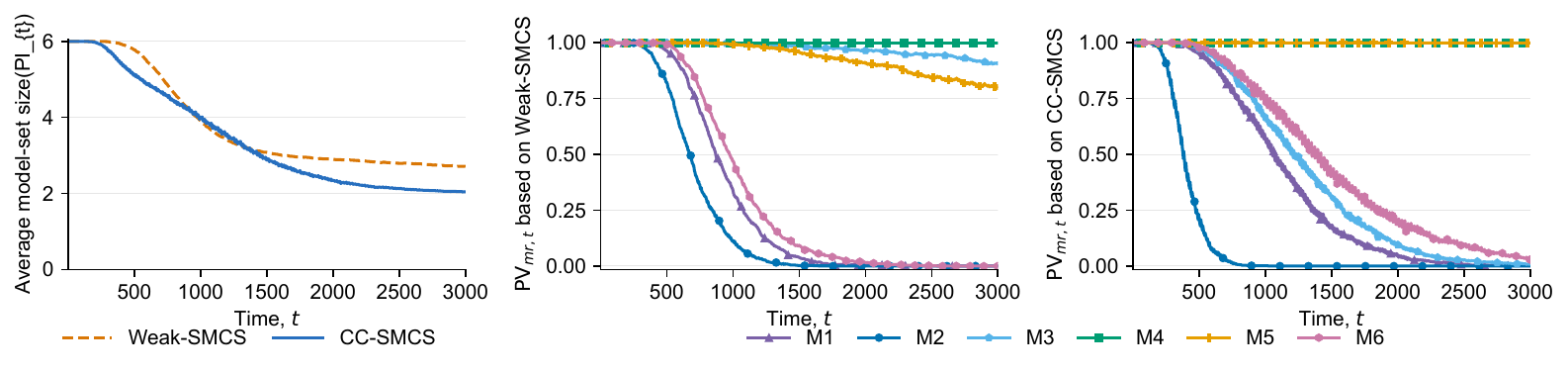}
\caption{Simulation 1}
\label{fig:sim1_025}
\end{subfigure}
\begin{subfigure}{\linewidth}
\centering
\includegraphics[width=\linewidth]{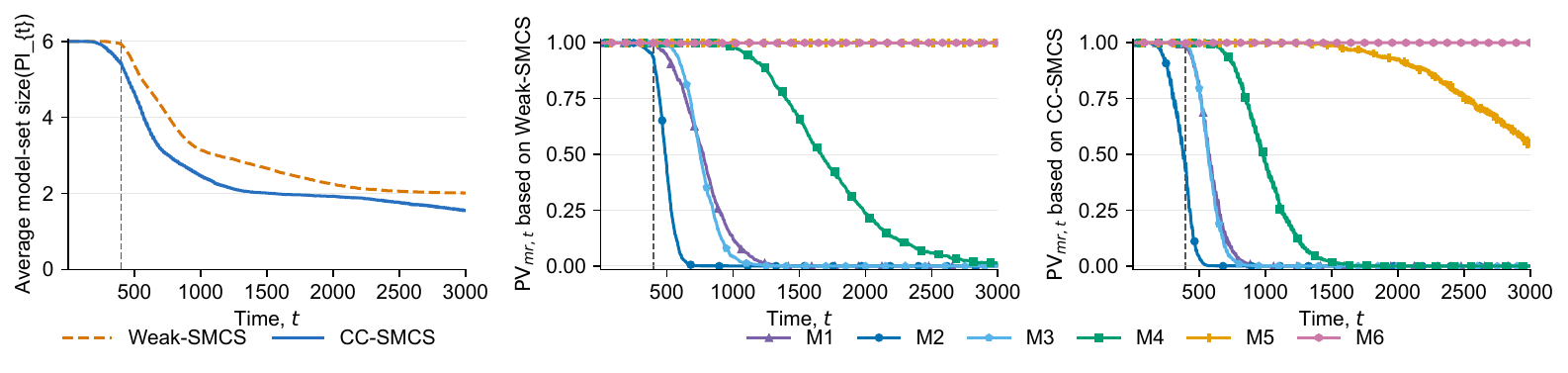}
\caption{Simulation 2}
\label{fig:sim2_025}
\end{subfigure}
\begin{subfigure}{\linewidth}
\centering
\includegraphics[width=\linewidth]{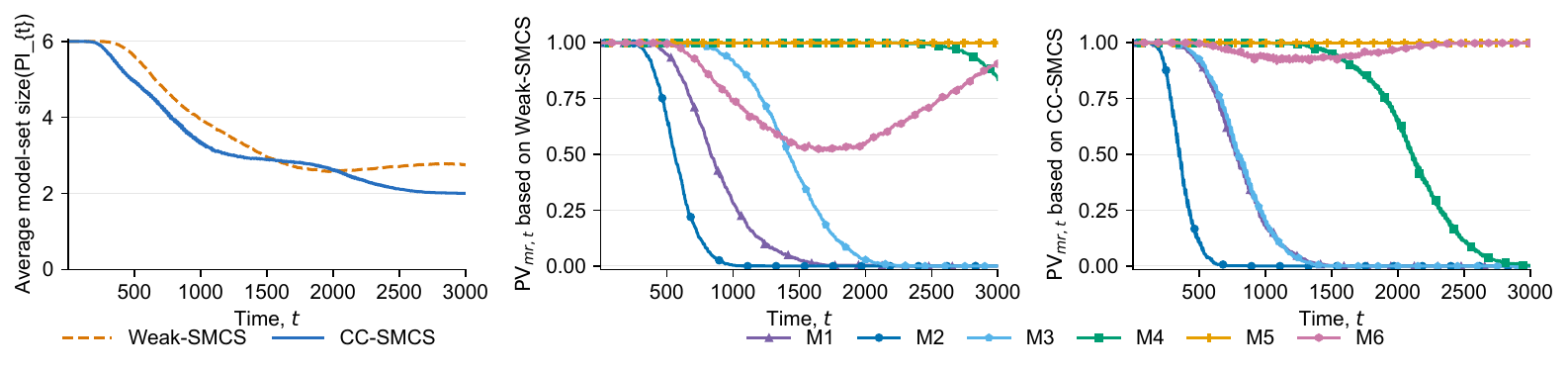}
\caption{Simulation 3}
\label{fig:sim3_025}
\end{subfigure}
\caption{The performance of simulations with $\delta=0.25$. The vertical dashed line in Figure~\ref{fig:sim2_025} indicates the structural break.}
\label{fig:sim_025}
\end{figure}

\section{Additional Details of the Empirical Study}
\label{sec:add_exper}
\subsection{Datasets}
\label{sub:data}
\paragraph{French electricity price}
The first dataset is the French electricity price series used by \citet{zaffran2022adaptive}, covering the period from January 8, 2016 to December 31, 2019. It contains 34,896 hourly observations.  We focus on the variable \texttt{Spot} and use only its historical values to construct one-step-ahead forecasts. It can be directly downloaded from \url{https://github.com/mzaffran/AdaptiveConformalPredictionsTimeSeries.}
\paragraph{Electricity load}
The second dataset is the electricity load series of \citet{hong2016probabilistic} from January 1, 2005, to September 30, 2010, containing 50,376 consecutive hourly observations, which can be found at \url{http://blog.drhongtao.com/2017/03/gefcom2014-load-forecasting-data.html}. This dataset was also used by \citet{lin2022conformal} for conformal prediction. Unlike their setting, we consider a univariate forecasting task and use only the historical values of \texttt{LOAD} to train the forecasting models and predict future load.

\paragraph{Weather}
The third dataset is the Weather dataset of \citet{wu2023timesnet}, and also used by \citet{chen2024conformalized}. It can be found at \url{https://huggingface.co/datasets/thuml/Time-Series-Library/blob/main/weather/weather.csv}, covering the year 2020 with observations recorded at 10-minute intervals. We use \texttt{T (degC)} as the target variable, which contains no missing values among the originally observed records. We find one duplicated timestamp and nine missing timestamps in the original time index. Since only nine timestamps are missing over the entire sample period, we remove the duplicated record, regularize the series to a complete 10-minute grid, and fill each missing value using the temperature observed at the same time on the previous day. This preprocessing preserves the regular sampling frequency while using only past information. The resulting series contains 52,704 regularly spaced observations. We then use only the historical values of \texttt{T (degC)} for model training and one-step-ahead prediction.

\paragraph{Wind power generation}
The final dataset is the Hackberry Wind Farm dataset used by \citet{xu2021enbpi}, which can be downloaded from \url{https://github.com/hamrel-cxu/EnbPI/blob/main/Data/Wind_Hackberry_Generation_2019_2020.csv}. While their original experiment focuses on the 2019 sample, we use the extended dataset from January 1, 2019, to July 31, 2020, containing 13,871 hourly observations. The original records preserve local daylight-saving-time transitions, which explains why the number of recorded observations differs from the number obtained by simply multiplying the number of calendar days by 24. We focus on the univariate variable \texttt{MWH} and use only its historical values for model training and one-step-ahead forecasting.

\subsection{Point Forecasting}
\label{sub:point_forecest}
\paragraph{Forecasting framework}
To construct conformal prediction intervals, we first train point forecasting models to obtain one-step-ahead forecasts. We consider three representative time-series forecasting backbones: DLinear \citep{zeng2023transformers}, PatchTST \citep{nie2023patchtst}, and TimesNet \citep{wu2023timesnet}. For the common point-forecasting framework, observations in each dataset are divided chronologically according to a $12:4:2:2$ ratio, corresponding approximately to $60\%$ for training, $20\%$ for validation, $10\%$ for conformal calibration, and the final $10\%$ for out-of-sample evaluation. This design is adapted from the $12:4:4$ training--validation--test ratio commonly used in long-term forecasting benchmarks, with the original test portion further divided equally into calibration and evaluation samples. The method-specific adaptations for CQR and EnbPI are described below.

All three models perform one-step-ahead univariate forecasting using only historical observations of the corresponding target variable. For a target observation at time $t$, the input window contains observations strictly before $t$. For the common point-forecasting framework, each series is standardized using the mean and standard deviation estimated from the training sample only. The same transformation is subsequently applied to the validation, calibration, and test samples, and all predictions are transformed back to the original scale before interval construction and evaluation.

\paragraph{Model configurations}
For the model-specific hyperparameters, we adopt the benchmark configurations reported in the original papers and their official implementations of DLinear \citep{zeng2023transformers}, PatchTST \citep{nie2023patchtst}, and TimesNet \citep{wu2023timesnet}, while changing the forecasting horizon to one to match our sequential prediction setting. DLinear uses an input sequence length of 336 and a moving-average kernel size of 25, with a batch size of 32, an initial learning rate of $5\times10^{-3}$, and zero weight decay. PatchTST utilizes an input sequence length of 336, a patch length of 16, a stride of 8, three encoder layers, four attention heads, $d_{\mathrm{model}}=16$, $d_{\mathrm{ff}}=128$, dropout of 0.3, fully connected dropout of 0.3, and RevIN normalization with the affine transformation disabled. Its batch size is 128, and its initial learning rate is $10^{-4}$. TimesNet adopts an input sequence length of 96, two TimesBlocks, $d_{\mathrm{model}}=16$, $d_{\mathrm{ff}}=32$, the top five detected periods, six inception kernels, and dropout of 0.1, with a batch size of 32 and an initial learning rate of $10^{-4}$.

\paragraph{Training and model selection}
All point forecasting models are optimized using the Adam optimizer with the mean squared error (MSE) loss and random seed 2021. Model selection is based exclusively on the validation sample, and the model parameters corresponding to the lowest validation MSE are retained for subsequent prediction. DLinear and TimesNet are trained for at most 10 epochs with an early-stopping patience of three epochs, while PatchTST is trained for at most 100 epochs with a patience of 100 epochs. We use type-1 learning-rate adjustment for DLinear and TimesNet and type-3 adjustment for PatchTST. In each case, the learning rate is updated at the end of each training epoch and the adjusted value is used in the following epoch. The same model-specific training configuration is applied across all four empirical datasets.

\paragraph{Adaptations for CQR and EnbPI}
Two interval construction methods require modifications to the common point-forecasting framework. For CQR \citep{romano2019cqr}, as detailed in Subsection~\ref{sub:cpm}, DLinear, PatchTST, and TimesNet are retained as forecasting backbones, while their scalar output layers are modified to produce two quantile outputs. For each backbone, a separate CQR model is trained for the $95\%$ nominal coverage level, using the sum of the lower- and upper-quantile pinball losses instead of MSE. EnbPI \citep{xu2021enbpi}, also detailed in Subsection~\ref{sub:cpm}, trains each bootstrap model independently, with early stopping determined by the training MSE on its own bootstrap sample rather than by validation loss.

\subsection{Conformal Prediction Methods}
\label{sub:cpm}

We consider six conformal prediction methods with nominal coverage levels of $95\%$. They differ in how the conformity scores are constructed, how the calibration information is weighted or updated, and whether the prediction interval is centered at a point forecast or obtained from conditional quantile estimates. Let $n_{\mathrm{cal}}$ denote the number of observations in the designated calibration sample.

\paragraph{Split Conformal Prediction}
The first method is classical split conformal prediction (Split CP) \citep{lei2018distribution}. Given a fitted point forecasting model with prediction $\widehat{Y}_i$, we use the absolute prediction error as the conformity score,
\[
S_i=\left|Y_i-\widehat{Y}_i\right|.
\]
For a calibration sample of size $n_{\mathrm{cal}}$, let $k=\lceil(n_{\mathrm{cal}}+1)(1-\alpha)\rceil$, and let $\widehat{q}_{1-\alpha}$ be the $k$th order statistic of the calibration scores. The prediction interval at time $t$ is
\[
\widehat{C}^{\mathrm{Split}}_t=\left[\widehat{Y}_t-\widehat{q}_{1-\alpha},\,\widehat{Y}_t+\widehat{q}_{1-\alpha}\right].
\]
The calibration scores and the resulting radius remain fixed throughout the entire test period. Therefore, newly observed test residuals do not affect subsequent Split CP intervals.

\paragraph{Weighted Conformal Prediction}
Weighted conformal prediction (WCP) \citep{barber2023conformal} modifies the conformal quantile by assigning different weights to past conformity scores. We use the same absolute prediction errors as in Split CP but assign decaying recency weights exponentially:
\[
w_{i,t}=\lambda^{t-i},\qquad \lambda=0.99.
\]
Thus, more recent observations receive greater weight. Following the weighted conformal construction, the test point is assigned unit mass at infinity when determining the weighted $(1-\alpha)$ quantile. Let $\widehat{q}^{\,w}_{1-\alpha,t}$ denote the resulting weighted conformal quantile. The WCP interval is
\[
\widehat{C}^{\mathrm{WCP}}_t=\left[\widehat{Y}_t-\widehat{q}^{\,w}_{1-\alpha,t},\,\widehat{Y}_t+\widehat{q}^{\,w}_{1-\alpha,t}\right].
\]
WCP utilizes a rolling residual window of fixed length. At the beginning of the test period, the residual window consists of the $n_{\mathrm{cal}}$ calibration residuals, and after $Y_t$ is revealed, the new absolute prediction residual is inserted, and only the most recent $n_{\mathrm{cal}}$ residuals are retained. 

\paragraph{Adaptive Conformal Inference}
Adaptive Conformal Inference (ACI) \citep{gibbs2021adaptive} dynamically adjusts the effective miscoverage level according to previous coverage errors. Let $\alpha$ denote the target miscoverage rate and $\alpha_t$ the level used to construct the interval at time $t$. Starting from $\alpha_1=\alpha$, ACI updates
\begin{equation}
\label{eq:gamma}
\alpha_{t+1}=\alpha_t+\gamma\left[\alpha-\mathbf{1}\left\{Y_t\notin\widehat{C}_t(\alpha_t)\right\}\right],\qquad \gamma=0.01.
\end{equation}
At time $t$, the empirical $(1-\alpha_t)$ quantile of the current absolute-residual window determines the interval radius. If $Y_t$ is not covered, $\alpha_t$ decreases, and the subsequent interval becomes more conservative; if $Y_t$ is covered, $\alpha_t$ increases. As in our WCP implementation, the residual window is initialized using the calibration residuals and updated after each observed test response while retaining only the most recent $n_{\mathrm{cal}}$ residuals. We do not truncate the recursive value of $\alpha_t$ to $[0,1]$; when the implied quantile level falls outside this range, we use the corresponding maximum or minimum residual in the current window, resulting in a finite-radius implementation of ACI.

\paragraph{Aggregated Adaptive Conformal Inference}
Aggregated Adaptive Conformal Inference (AgACI) \citep{zaffran2022adaptive} reduces the sensitivity of ACI to the choice of a single learning rate by running multiple ACI experts in parallel. We employ the 30 learning rates adopted in the official AgACI implementation. Each $\gamma\in\Gamma$ defines one ACI expert that updates its adaptive level $\alpha_{t,\gamma}$ according to Equation~(\ref{eq:gamma}) and consequently produces its own lower and upper prediction endpoints. Rather than directly averaging these intervals, the lower and upper endpoints are aggregated separately using Bernstein Online Aggregation (BOA). The BOA updates are based on pinball-loss gradients at levels $\alpha/2$ for the lower endpoint and $1-\alpha/2$ for the upper endpoint. After $Y_t$ is revealed, each ACI expert updates its adaptive level, and the common absolute-residual window is updated while retaining the most recent $n_{\mathrm{cal}}$ observations. The same finite-radius convention described above is applied to each ACI expert when its implied quantile level falls outside $[0,1]$.

\paragraph{Conformalized Quantile Regression}
Conformalized Quantile Regression (CQR) \citep{romano2019cqr} first estimates conditional quantiles rather than constructing a symmetric interval around a point forecast. For nominal coverage $1-\alpha$, the forecasting backbone jointly estimates
\[
\widehat{q}_{\alpha/2}(X_t)\quad\text{and}\quad\widehat{q}_{1-\alpha/2}(X_t),
\]
using the pinball loss. The corresponding raw quantile interval is
\[
\left[\widehat{q}_{\alpha/2}(X_t),\,\widehat{q}_{1-\alpha/2}(X_t)\right].
\]
For each observation in the calibration sample, the CQR conformity score is
\[
S_i=\max\left\{\widehat{q}_{\alpha/2}(X_i)-Y_i,\,Y_i-\widehat{q}_{1-\alpha/2}(X_i)\right\}.
\]
Let $\widehat{q}^{\mathrm{CQR}}_{1-\alpha}$ denote the finite-sample conformal quantile of these scores. The final interval is
\[
\widehat{C}^{\mathrm{CQR}}_t=\left[\widehat{q}_{\alpha/2}(X_t)-\widehat{q}^{\mathrm{CQR}}_{1-\alpha},\,\widehat{q}_{1-\alpha/2}(X_t)+\widehat{q}^{\mathrm{CQR}}_{1-\alpha}\right].
\]
For $95\%$ nominal coverage, we jointly train the lower and upper quantile outputs at levels 0.025 and 0.975. DLinear, PatchTST, and TimesNet are used as the corresponding quantile-regression backbones with their scalar output heads replaced by two-quantile output heads. The two raw outputs are ordered at prediction time to avoid quantile crossing. The conformal correction is computed once from the fixed calibration sample and remains unchanged throughout the test period. Since the CQR conformity scores are signed relative to the raw interval boundaries, the resulting conformal correction may be negative and can therefore contract an initially conservative quantile interval. If such contraction makes the two final endpoints cross, we project the resulting empty interval to its midpoint rather than exchanging the lower and upper endpoints.

\paragraph{Ensemble Batch Prediction Intervals}
Finally, Ensemble Batch Prediction Intervals (EnbPI) \citep{xu2021enbpi} constructs sequential prediction intervals through bootstrap ensemble forecasting and out-of-bag (OOB) residuals without requiring a separate conformal calibration set. We use $B=30$ bootstrap models, mean aggregation, and batch size $s=1$. Let $S_b$ denote the bootstrap sample used to train model $\widehat{f}^{\,b}$. For each observation $i$ in the initial EnbPI sample, only bootstrap models for which $i\notin S_b$ are aggregated to form its OOB prediction,
\[
\widehat{f}_{-i}(X_i)=\frac{1}{|\mathcal{B}_i|}\sum_{b\in\mathcal{B}_i}\widehat{f}^{\,b}(X_i),\qquad \mathcal{B}_i=\{b:i\notin S_b\}.
\]
The corresponding absolute OOB residual is
\[
S_i^{\mathrm{OOB}}=\left|Y_i-\widehat{f}_{-i}(X_i)\right|.
\]
For a future input $X_t$, the same OOB model sets are used to produce leave-one-out ensemble predictions $\{\widehat{f}_{-i}(X_t)\}$ across the initial observations. Following Algorithm~1 of \citet{xu2021enbpi}, the prediction center is the empirical $(1-\alpha)$ quantile of these leave-one-out ensemble predictions,
\[
\widehat{Y}^{\mathrm{EnbPI}}_t=\widehat{Q}_{1-\alpha}\left(\{\widehat{f}_{-i}(X_t)\}_{i=1}^{T_0}\right),
\]
where $T_0$ denotes the number of initial EnbPI prediction--response pairs. The interval radius is obtained from the empirical $(1-\alpha)$ quantile of the current absolute-residual window,
\[
w_t=\widehat{Q}_{1-\alpha}\left(\{S_i^{\mathrm{OOB}}\}\right)
\]
and the resulting symmetric prediction interval is
\[
\widehat{C}^{\mathrm{EnbPI}}_t=\left[\widehat{Y}^{\mathrm{EnbPI}}_t-w_t,\,\widehat{Y}^{\mathrm{EnbPI}}_t+w_t\right].
\]
Because $s=1$, once $Y_t$ is revealed, the new absolute prediction residual is inserted into the residual window and the oldest residual is removed, so the residual distribution is updated after every test observation.

Unlike Split CP, WCP, ACI, AgACI, and CQR, EnbPI does not use the designated 10\% calibration segment as a separate conformal calibration sample and does not require a separate validation set for training its bootstrap models. Instead, each bootstrap model is trained independently, with early stopping determined by the training MSE on its own bootstrap sample. To ensure an identical evaluation horizon across all six methods, the initial EnbPI history is constructed entirely from observations available before the common test period and has the same raw length as the training portion used in the corresponding point forecasting experiment. The maximum number of epochs and the early-stopping patience are set according to the corresponding forecasting backbone.

\subsection{Additional results }
\label{sub:add_re}
This section reports the remaining empirical results for the four datasets under DLinear, PatchTST, and TimesNet with $\delta\in\{0.10,0.25\}$. Throughout the empirical analysis, we set $R=1$, $a_{rt}=1$, and $\epsilon_r=0$, so that $\tau_r=\alpha_r=0.05$. For the six candidate methods, we use the balanced allocation $\eta_{mr}=\delta/6$. The values of the mixture tuning parameter $\rho^\star$ are reported in Table~\ref{tab:choe_rho}. As in the main analysis, we examine cumulative coverage, the ratio $\widehat{\Gamma}_{mr,t}/\mathrm{rad}_{mr,t}$, and the cumulative mean interval width of each method to explain the evolution of the CC-SMCS model set over time.

For the French electricity price dataset, PatchTST and TimesNet exhibit patterns broadly similar to those observed under DLinear. We focus on the time at which the certified set first becomes nonempty. Under PatchTST, Split CP and EnbPI are the first methods to enter the certified set because their cumulative coverage performance becomes sufficiently favorable. Since Split CP has a larger cumulative mean interval width than EnbPI, it is excluded from the final model set earlier. As the cumulative coverage performance of CQR improves, CQR also enters the certified set, but is subsequently excluded because of its relatively large interval width. CC-SMCS also allows a method to re-enter the final model set when its relative efficiency improves. WCP provides such an example: while remaining in the possible set, it is excluded when its cumulative mean interval width exceeds the minimum cumulative mean width among the certified methods, $q_t^{\mathrm{cert}}$, and is re-included once its cumulative mean interval width falls below $q_t^{\mathrm{cert}}$ (Figures~\ref{fig:france_patch_010} and~\ref{fig:france_patch_025}). Similar patterns are observed for DLinear with $\delta=0.25$ (Figure~\ref{fig:france_025}) and for TimesNet under both $\delta=0.10$ and $\delta=0.25$ (Figures~\ref{fig:france_timesnet_010} and~\ref{fig:france_timesnet_025}).

For the Electricity Load dataset, the results under DLinear with $\delta=0.25$ are similar to those observed for the French electricity price dataset (Figure~\ref{fig:load_025}). TimesNet also exhibits similar dynamics under both $\delta=0.10$ and $\delta=0.25$, with the model set evolving as the coverage performance and interval widths of the candidate methods change over time (Figures~\ref{fig:load_timesnet_010} and~\ref{fig:load_timesnet_025}). In contrast, a different pattern emerges under PatchTST. CQR becomes infeasible because its cumulative coverage remains substantially below the nominal level of $95\%$, causing $\widehat{\Gamma}_{mr,t}/\mathrm{rad}_{mr,t}$ to exceed 1. As a result, CQR leaves the possible set and is therefore excluded from the final model set (Figures~\ref{fig:load_patch_010} and~\ref{fig:load_patch_025}).

For the Weather dataset, CC-SMCS generally excludes methods with relatively wide intervals once a narrower method becomes certified feasible, while retaining methods with adequate coverage and relatively narrow intervals. This pattern is consistent across DLinear, PatchTST, and TimesNet, as shown in Figures~\ref{fig:weather_patch_010}, \ref{fig:weather_patch_025}, \ref{fig:weather_timesnet_010}, and~\ref{fig:weather_timesnet_025}.

For the Wind Power dataset, the model-set paths differ across the three forecasting backbones. Under DLinear and PatchTST, the certified set remains empty until near the end of the prediction period. Split CP is the first method to enter the certified set but also has the largest cumulative mean interval width, and it is excluded only after EnbPI becomes certified with a smaller width (Figures~\ref{fig:wind_025}, \ref{fig:wind_patch_010}, and~\ref{fig:wind_patch_025}). This reflects relatively weak separation in the coverage--width profiles of the candidate pipelines, so the available evidence is insufficient to reduce the model set earlier. In contrast, under TimesNet, Split CP and EnbPI become certified earlier and the differences in interval width are more pronounced, leading to earlier exclusions of Split CP and CQR (Figures~\ref{fig:wind_timesnet_010} and~\ref{fig:wind_timesnet_025}). The contrast across forecasting backbones illustrates that the informativeness of CC-SMCS depends not only on the conformal calibration method but also on the underlying forecasting backbone, since changes in the base forecaster can alter the coverage--efficiency profiles of the resulting candidate pipelines.

\begin{figure}[htbp]
\centering
\begin{subfigure}{\linewidth}
\centering
\includegraphics[width=\linewidth]{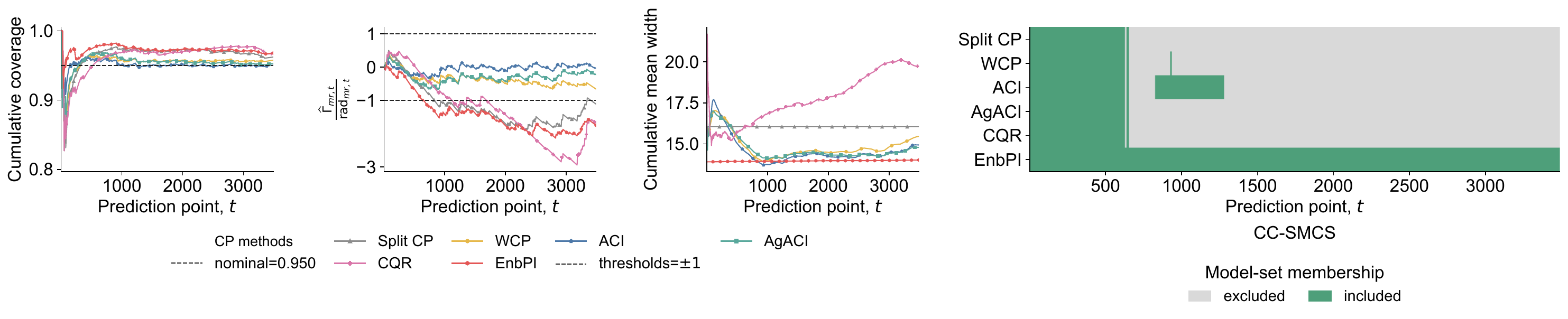}
\caption{French Electricity Price}
\label{fig:france_025}
\end{subfigure}

\vspace{0.1cm}
\begin{subfigure}{\linewidth}
\centering
\includegraphics[width=\linewidth]{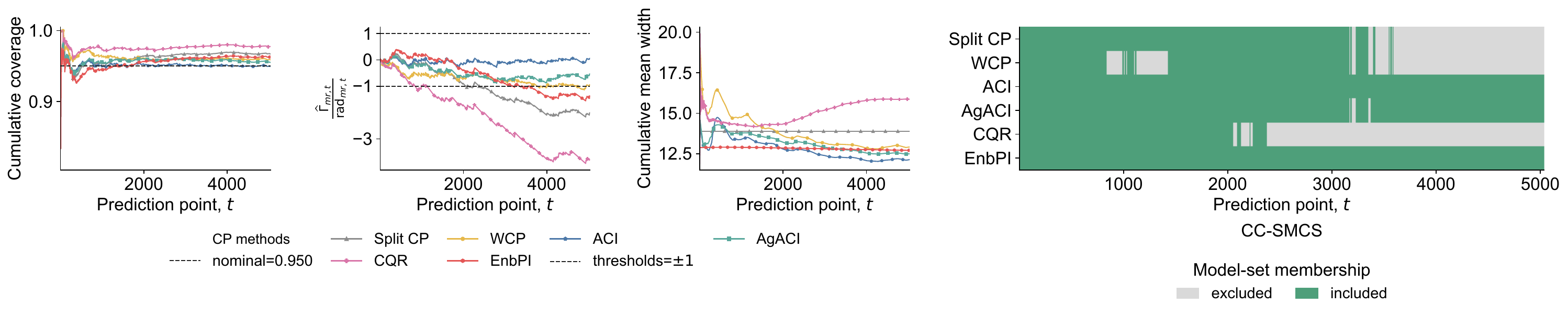}
\caption{Electricity Load}
\label{fig:load_025}
\end{subfigure}

\begin{subfigure}{\linewidth}
\centering
\includegraphics[width=\linewidth]{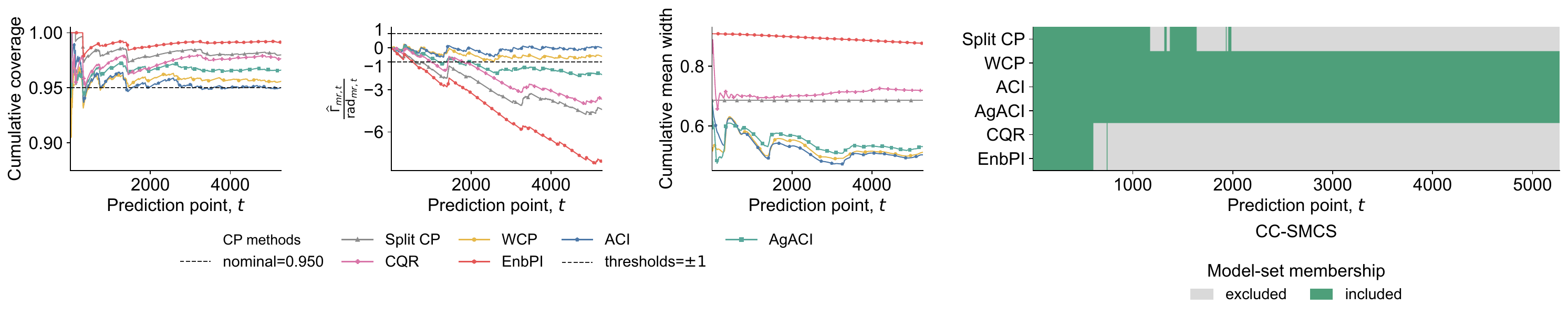}
\caption{Weather}
\label{fig:weather_025}
\end{subfigure}

\begin{subfigure}{\linewidth}
\centering
\includegraphics[width=\linewidth]{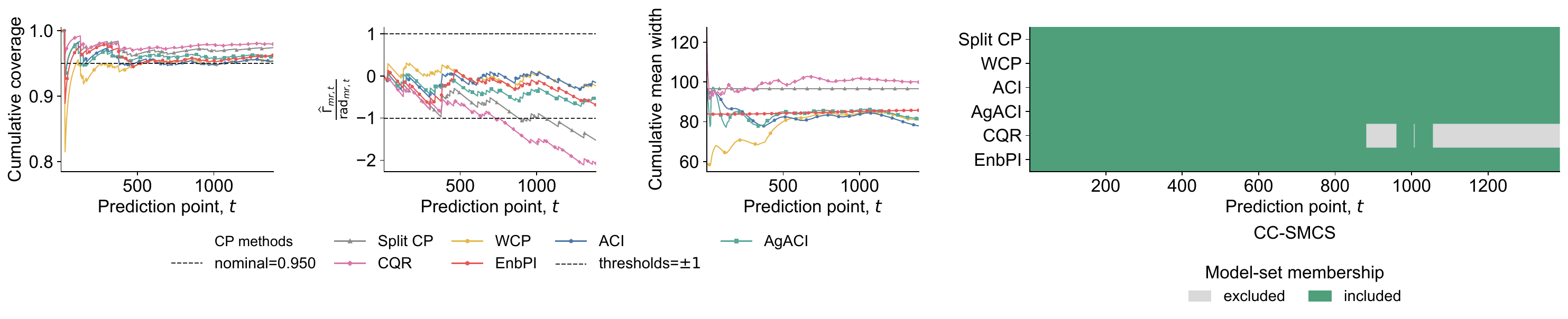}
\caption{Wind Power}
\label{fig:wind_025}
\end{subfigure}
\caption{Empirical results for DLinear at 95\% nominal coverage with $\delta=0.25$.}
\label{fig:re_dlinear_025}
\end{figure}

\begin{figure}[htbp]
\centering
\begin{subfigure}{\linewidth}
\centering
\includegraphics[width=\linewidth]{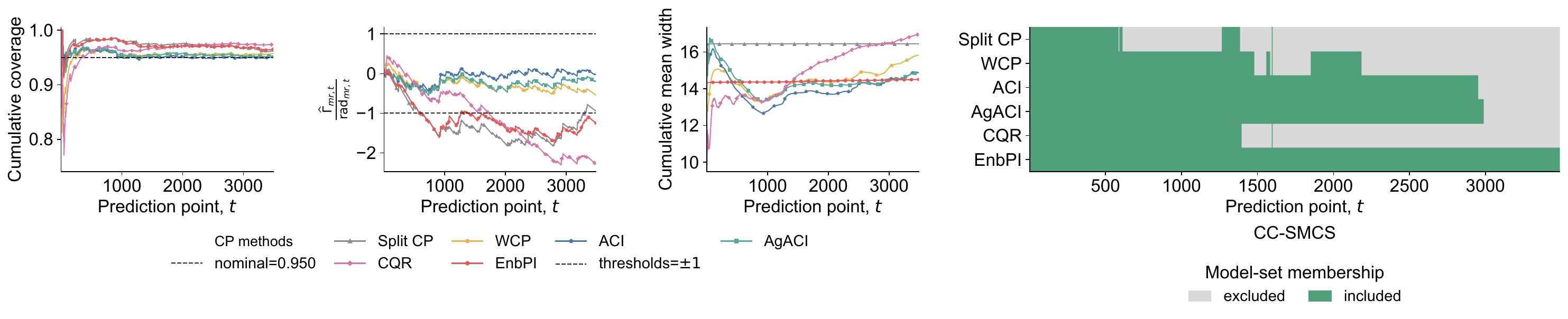}
\caption{French Electricity Price}
\label{fig:france_patch_010}
\end{subfigure}

\vspace{0.1cm}
\begin{subfigure}{\linewidth}
\centering
\includegraphics[width=\linewidth]{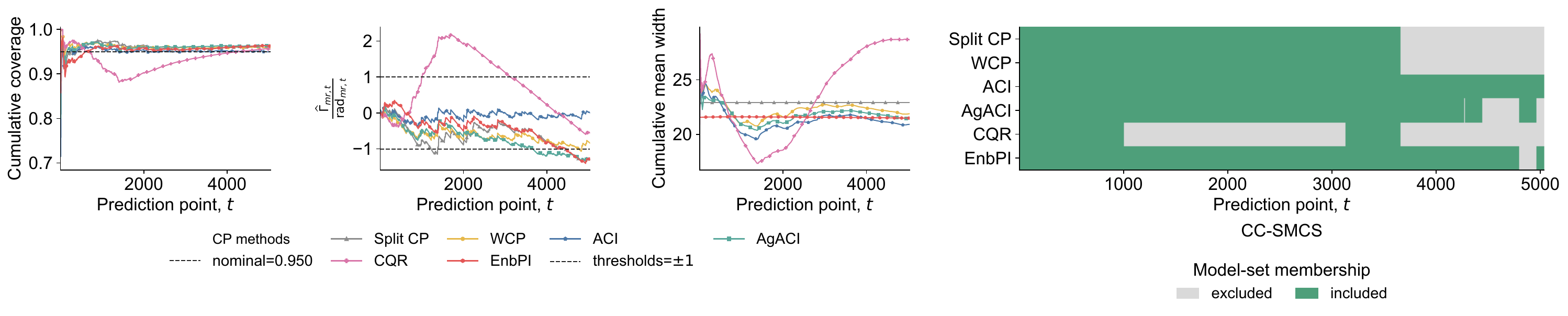}
\caption{Electricity Load}
\label{fig:load_patch_010}
\end{subfigure}

\begin{subfigure}{\linewidth}
\centering
\includegraphics[width=\linewidth]{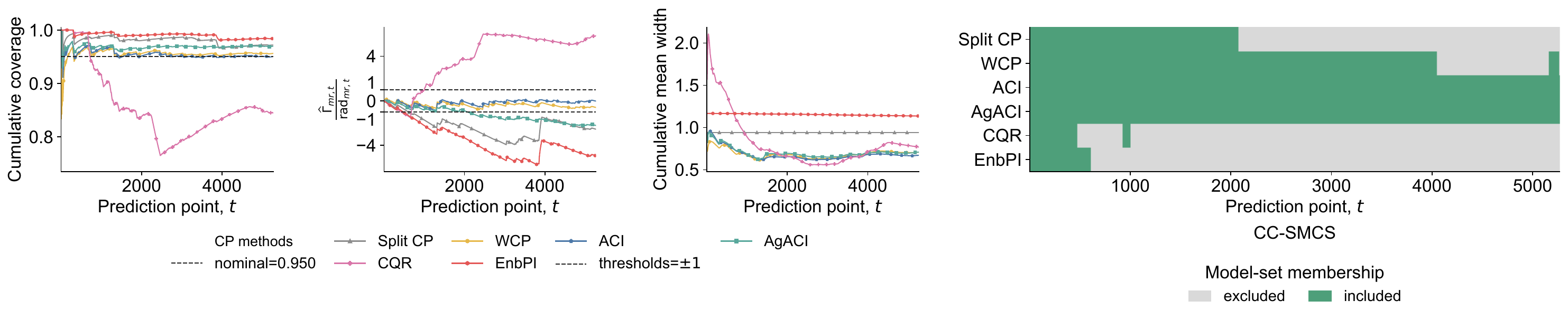}
\caption{Weather}
\label{fig:weather_patch_010}
\end{subfigure}

\begin{subfigure}{\linewidth}
\centering
\includegraphics[width=\linewidth]{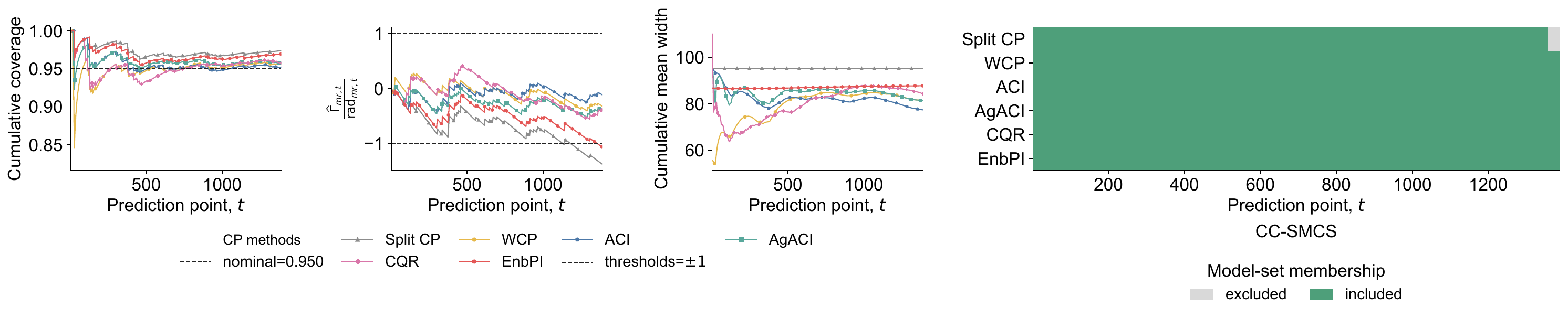}
\caption{Wind Power}
\label{fig:wind_patch_010}
\end{subfigure}
\caption{Empirical results for PatchTST at 95\% nominal coverage with $\delta=0.10$.}
\label{fig:re_patch_010}
\end{figure}

\begin{figure}[htbp]
\centering
\begin{subfigure}{\linewidth}
\centering
\includegraphics[width=\linewidth]{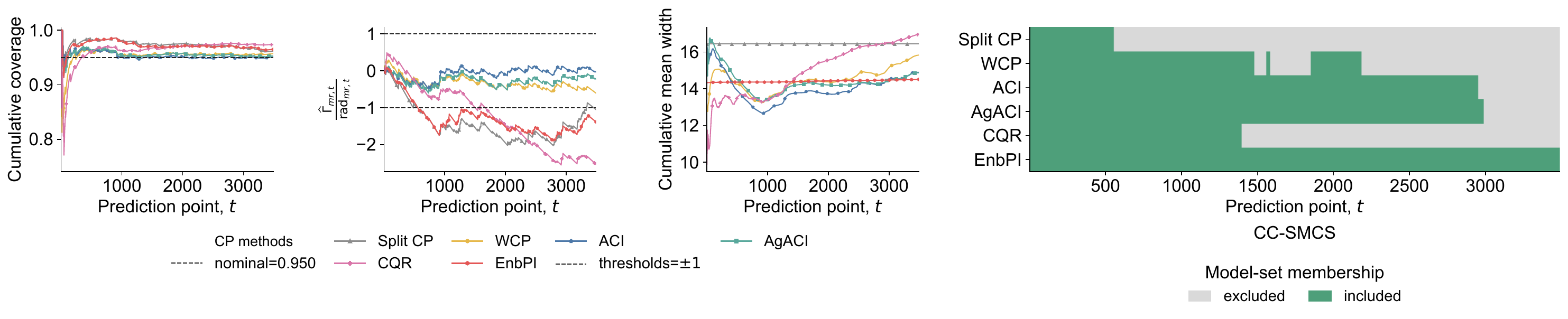}
\caption{French Electricity Price}
\label{fig:france_patch_025}
\end{subfigure}

\vspace{0.1cm}
\begin{subfigure}{\linewidth}
\centering
\includegraphics[width=\linewidth]{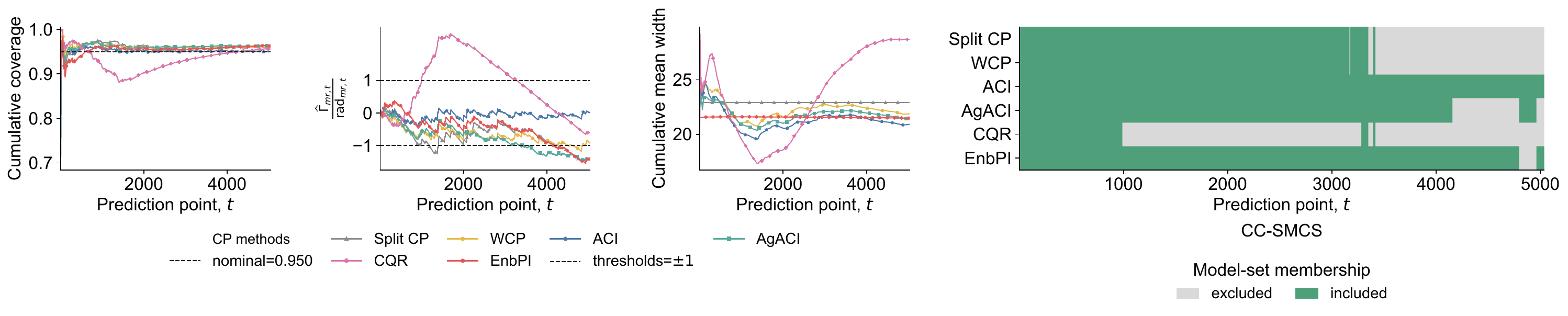}
\caption{Electricity Load}
\label{fig:load_patch_025}
\end{subfigure}

\begin{subfigure}{\linewidth}
\centering
\includegraphics[width=\linewidth]{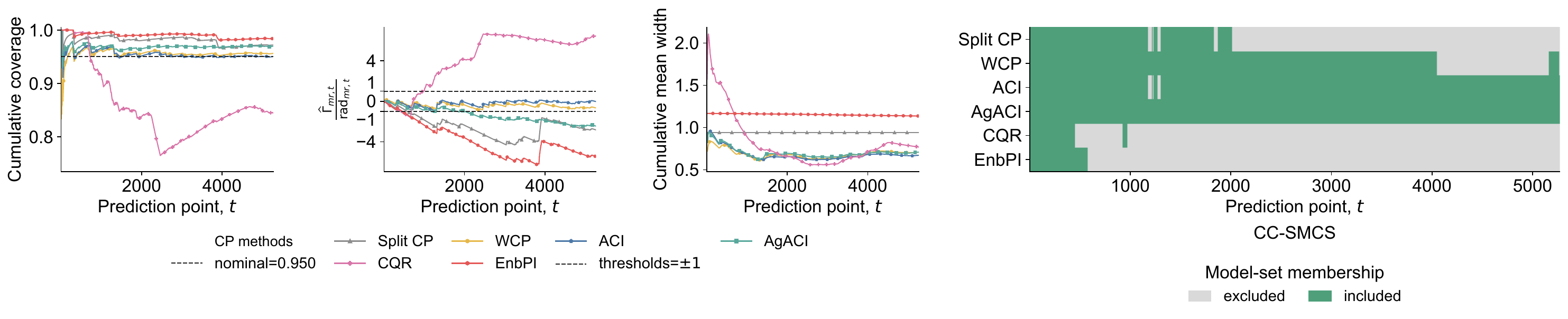}
\caption{Weather}
\label{fig:weather_patch_025}
\end{subfigure}

\begin{subfigure}{\linewidth}
\centering
\includegraphics[width=\linewidth]{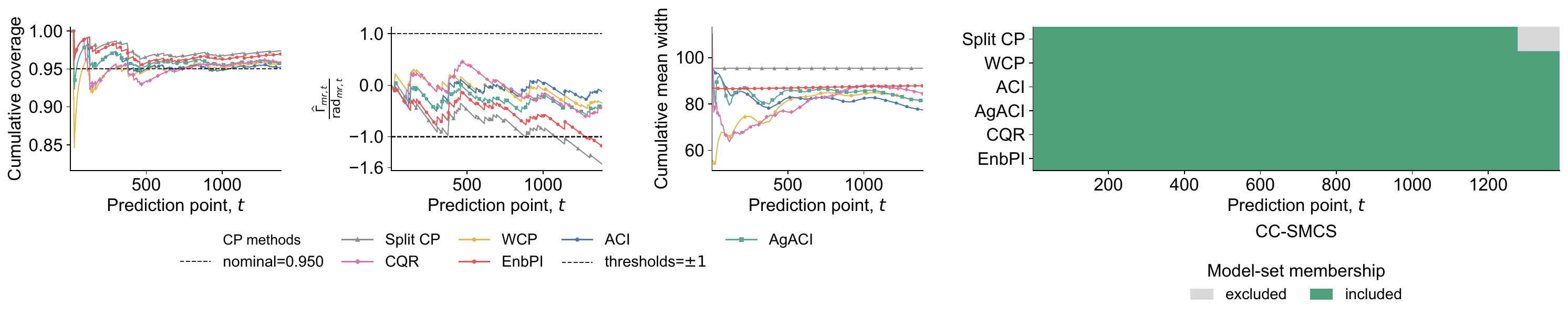}
\caption{Wind Power}
\label{fig:wind_patch_025}
\end{subfigure}
\caption{Empirical results for PatchTST at 95\% nominal coverage with $\delta=0.25$.}
\label{fig:re_patch_025}
\end{figure}

\begin{figure}[htbp]
\centering
\begin{subfigure}{\linewidth}
\centering
\includegraphics[width=\linewidth]{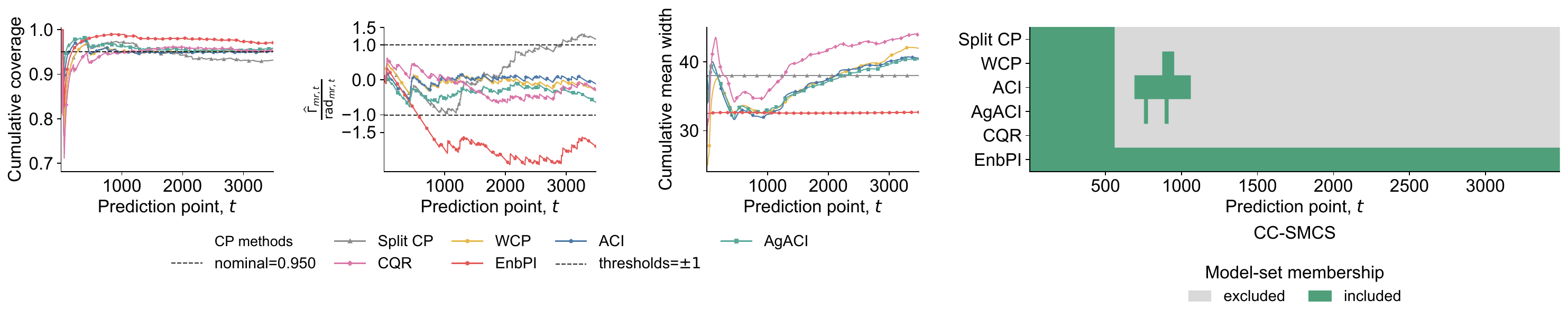}
\caption{French Electricity Price}
\label{fig:france_timesnet_010}
\end{subfigure}

\vspace{0.1cm}
\begin{subfigure}{\linewidth}
\centering
\includegraphics[width=\linewidth]{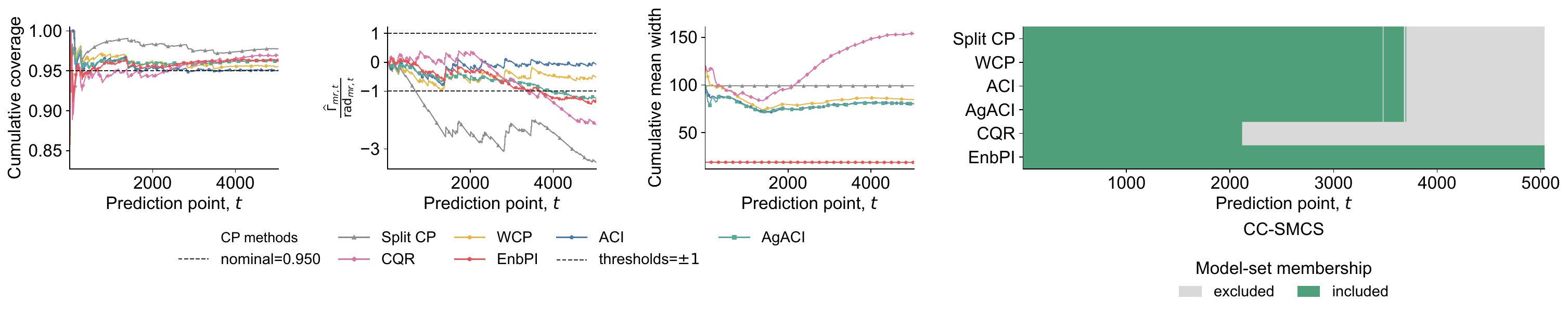}
\caption{Electricity Load}
\label{fig:load_timesnet_010}
\end{subfigure}

\begin{subfigure}{\linewidth}
\centering
\includegraphics[width=\linewidth]{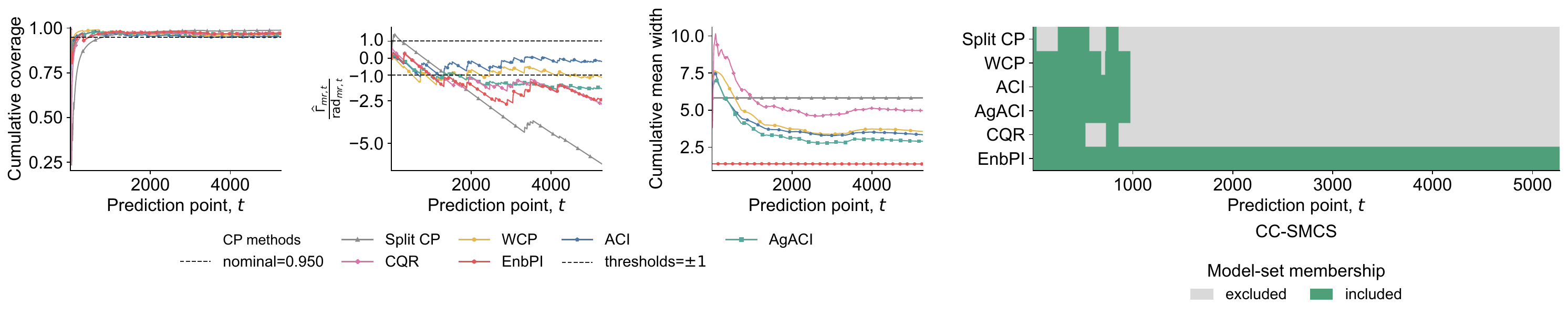}
\caption{Weather}
\label{fig:weather_timesnet_010}
\end{subfigure}

\begin{subfigure}{\linewidth}
\centering
\includegraphics[width=\linewidth]{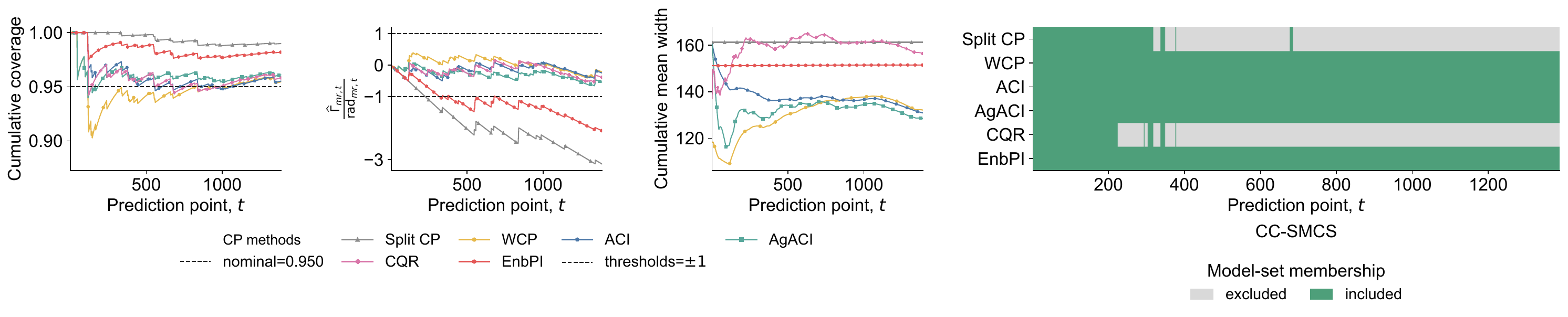}
\caption{Wind Power}
\label{fig:wind_timesnet_010}
\end{subfigure}
\caption{Empirical results for TimesNet at 95\% nominal coverage with $\delta=0.10$.}
\label{fig:re_timesnet_010}
\end{figure}

\begin{figure}[htbp]
\centering
\begin{subfigure}{\linewidth}
\centering
\includegraphics[width=\linewidth]{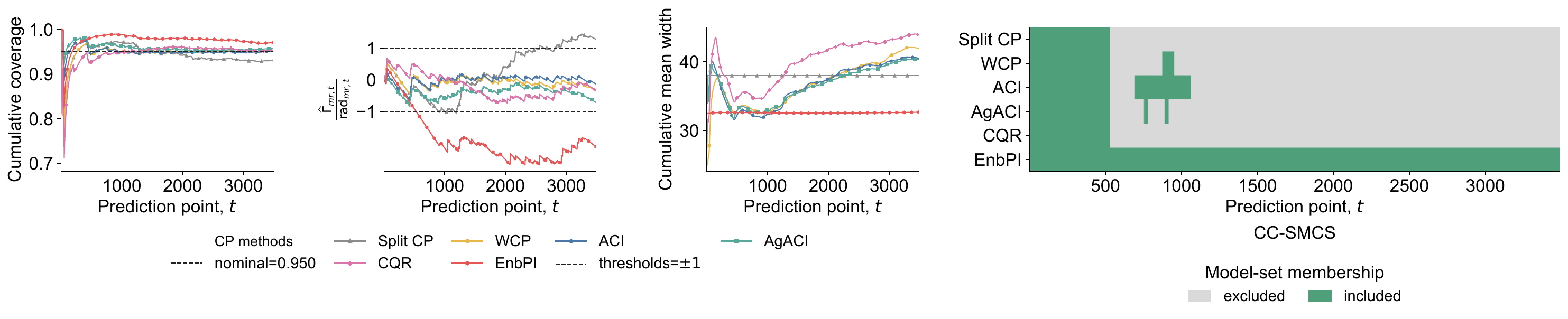}
\caption{French Electricity Price}
\label{fig:france_timesnet_025}
\end{subfigure}

\vspace{0.1cm}
\begin{subfigure}{\linewidth}
\centering
\includegraphics[width=\linewidth]{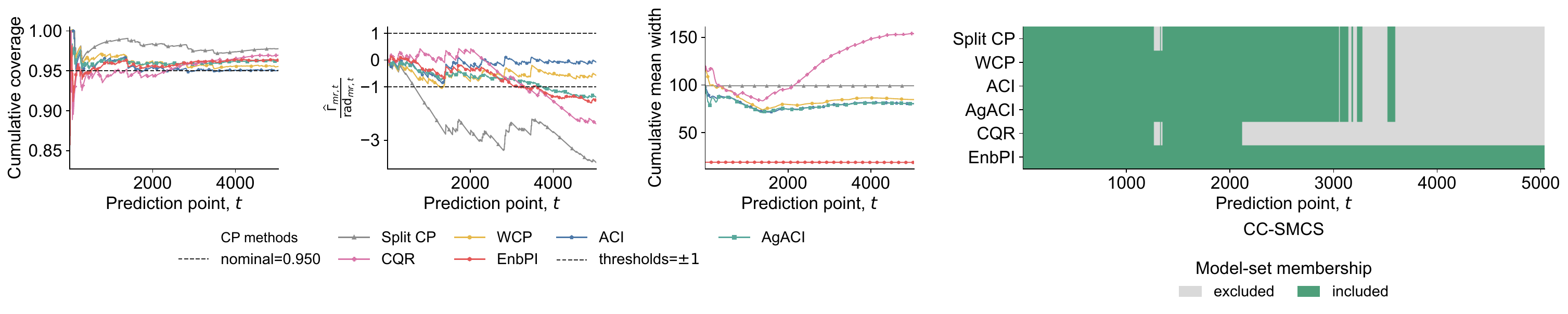}
\caption{Electricity Load}
\label{fig:load_timesnet_025}
\end{subfigure}

\begin{subfigure}{\linewidth}
\centering
\includegraphics[width=\linewidth]{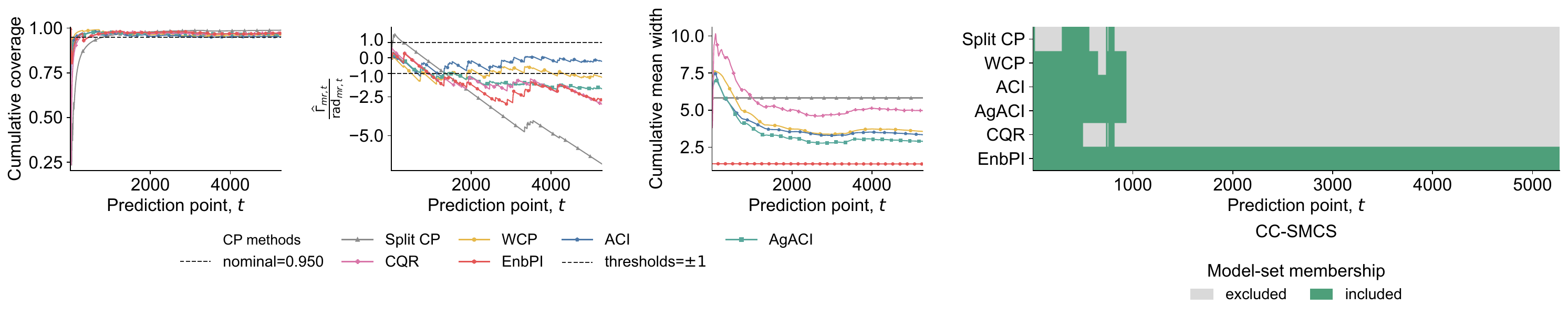}
\caption{Weather}
\label{fig:weather_timesnet_025}
\end{subfigure}

\begin{subfigure}{\linewidth}
\centering
\includegraphics[width=\linewidth]{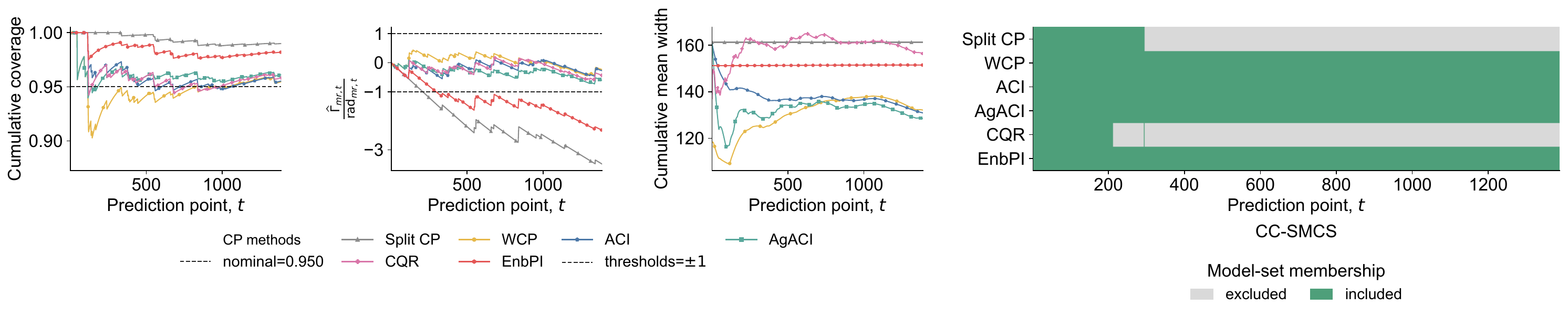}
\caption{Wind Power}
\label{fig:wind_timesnet_025}
\end{subfigure}
\caption{Empirical results for TimesNet at 95\% nominal coverage with $\delta=0.25$.}
\label{fig:re_timesnet_025}
\end{figure}

\section{More Discussions and Extensions}
\subsection{Delayed multi-horizon feedback}\label{app:delay}

For an $h$-step forecast issued at origin $s$, $C_{mrsh}$ is $\G_s$-measurable and is evaluated using $Y_{s+h}$. If its increment is inserted at calendar time $s+h$ into the usual filtration, conditioning on $\G_{s+h-1}$ need not recover the forecast-origin conditional mean given $\G_s$. We therefore partition forecast origins into residue classes
\begin{equation}
\Iset_{h,k}=\{k,k+h,k+2h,\ldots\},
\qquad k=1,\ldots,h.
\label{eq:interlace}
\end{equation}
Within each class, an increment issued at one origin is resolved by the next origin in the same class. We run a separate confidence sequence for every $(m,r,h,k)$ stream and allocate the error probability across all streams.

\begin{corollary}\label{cor:delay}
If the streamwise allocations satisfy $\sum_{m,r,h,k}\eta_{mrhk}\leq\delta$, then Theorem~\ref{thm:main} holds simultaneously over all calendar times and monitored horizons, with each horizon-specific target defined over forecast origins whose outcomes have been resolved.
\end{corollary}

Before proving Corollary \ref{cor:delay}, we reformulate the problem and derive the following preliminary results first. For each monitored horizon $h$, a forecast issued at origin $s$ is $\G_s$-measurable and targets $Y_{s+h}$. Define
\begin{align}
B_{mrsh}&:=\1\{Y_{s+h}\notin C_{mrsh}\},\qquad X_{mrsh}:=a_{rsh}(B_{mrsh}-\tau_{rh}),\\
\widehat B_{mrsh}&:=
\begin{cases}
\tau_{rh}, & A_{rh,s-h}=0,\\
\dfrac{\sum_{u=1}^{s-h}a_{ruh}B_{mruh}}{A_{rh,s-h}}, & A_{rh,s-h}>0,
\end{cases}\\
\widehat X_{mrsh}&:=a_{rsh}(\widehat B_{mrsh}-\tau_{rh}),\qquad \mu_{mrsh}:=\E[X_{mrsh}\mid\G_s],
\end{align}
where $a_{rsh}\in[0,1]$ is $\G_s$-measurable and $A_{rh,s-h}:=\sum_{u=1}^{s-h}a_{ruh}$, with sums over empty index sets understood as zero.

Fix $k\in\{1,\ldots,h\}$ and enumerate the origins in the $k$th residue stream by
\[
s_{k,n}:=k+nh,\qquad n=0,1,2,\ldots.
\]
Define
\[
\mathcal H_{k,n}:=\G_{s_{k,n}},\qquad D_{mrh,k,n+1}:=X_{mr,s_{k,n},h}-\mu_{mr,s_{k,n},h}.
\]
Since $s_{k,n}+h=s_{k,n+1}$, we have
\[
X_{mr,s_{k,n},h}\in\mathcal H_{k,n+1},\qquad \mu_{mr,s_{k,n},h}\in\mathcal H_{k,n},\qquad D_{mrh,k,n+1}\in\mathcal H_{k,n+1}.
\]
Moreover, since $\mathcal H_{k,n}=\G_{s_{k,n}}$ and $\mu_{mr,s_{k,n},h}=\E[X_{mr,s_{k,n},h}\mid\G_{s_{k,n}}]$,
\begin{align*}
\E[D_{mrh,k,n+1}\mid\mathcal H_{k,n}]
&=\E[X_{mr,s_{k,n},h}-\mu_{mr,s_{k,n},h}\mid\G_{s_{k,n}}]\\
&=\E[X_{mr,s_{k,n},h}\mid\G_{s_{k,n}}]-\mu_{mr,s_{k,n},h}\\
&=0.
\end{align*}
Conditionally on $\mathcal H_{k,n}$, define
\[
\Delta_{mrh,k,n+1}:=(X_{mr,s_{k,n},h}-\widehat X_{mr,s_{k,n},h})^2.
\]
Since $B_{mr,s_{k,n},h}\in\{0,1\}$, $\widehat B_{mr,s_{k,n},h}\in[0,1]$, and $a_{r,s_{k,n},h}\in[0,1]$,
\[
X_{mr,s_{k,n},h}-\widehat X_{mr,s_{k,n},h}\in[-a_{r,s_{k,n},h},a_{r,s_{k,n},h}]\subseteq[-1,1],
\]
By Lemma~\ref{lem:hoeffding}, for every $\lambda\in[0,1)$,
\[
\E\left[\left.\exp\{\lambda D_{mrh,k,n+1}-\psi_E(\lambda)\Delta_{mrh,k,n+1}\}\right|\mathcal H_{k,n}\right]\leq1.
\]
Thus, Lemma~\ref{lem:mixed_bound} applies to each residue stream with
\[
V_{mrh,k,N}:=\sum_{n=0}^{N-1}(X_{mr,s_{k,n},h}-\widehat X_{mr,s_{k,n},h})^2=\sum_{n=0}^{N-1}\Delta_{mrh,k,n+1}.
\]
For every $(m,r,h,k)$,
\[
\Pp\left(\forall N\geq1:\ \left|\sum_{n=0}^{N-1}D_{mrh,k,n+1}\right|\leq b_{\eta_{mrhk},\rho}(V_{mrh,k,N})\right)\geq1-\eta_{mrhk}.
\]

At calendar time $t$, define the number of resolved origins in stream $k$ by
\[
N_{h,k}(t):=\#\{n\geq0:s_{k,n}+h\leq t\}.
\]
With empty sums understood as zero, define
\begin{align*}
S_{mrh,k}(t)&:=\sum_{n=0}^{N_{h,k}(t)-1}X_{mr,s_{k,n},h},\\
A_{rh,k}(t)&:=\sum_{n=0}^{N_{h,k}(t)-1}a_{r,s_{k,n},h},\\
V_{mrh,k}(t)&:=\sum_{n=0}^{N_{h,k}(t)-1}(X_{mr,s_{k,n},h}-\widehat X_{mr,s_{k,n},h})^2.
\end{align*}
Taking $N=N_{h,k}(t)$ in the streamwise confidence sequence gives
\[
\left|S_{mrh,k}(t)-\sum_{n=0}^{N_{h,k}(t)-1}\mu_{mr,s_{k,n},h}\right|\leq b_{\eta_{mrhk},\rho}(V_{mrh,k}(t))
\]
for all stream resolution times; when $N_{h,k}(t)=0$, both sums are zero and the inequality holds trivially.

Let $\mathcal E_{mrhk}$ denote the time-uniform event for stream $(m,r,h,k)$. Then
\[
\Pp\left(\bigcap_{m,r,h,k}\mathcal E_{mrhk}\right)\geq1-\sum_{m,r,h,k}\Pp(\mathcal E_{mrhk}^c)\geq1-\sum_{m,r,h,k}\eta_{mrhk}\geq1-\delta.
\]

Define
\[
A_{rh}^{\mathrm{delay}}(t):=\sum_{k=1}^{h}A_{rh,k}(t).
\]
When $A_{rh}^{\mathrm{delay}}(t)>0$, define
\[
\Gamma_{mrh,t}^{\mathrm{delay}}:=\frac{\sum_{k=1}^{h}\sum_{n=0}^{N_{h,k}(t)-1}\mu_{mr,s_{k,n},h}}{A_{rh}^{\mathrm{delay}}(t)}.
\]
On the joint stream event, for every $k$,
\[
\left|S_{mrh,k}(t)-\sum_{n=0}^{N_{h,k}(t)-1}\mu_{mr,s_{k,n},h}\right|\leq b_{\eta_{mrhk},\rho}(V_{mrh,k}(t)).
\]
Therefore, by the triangle inequality,
\begin{align*}
\left|\sum_{k=1}^{h}S_{mrh,k}(t)-\sum_{k=1}^{h}\sum_{n=0}^{N_{h,k}(t)-1}\mu_{mr,s_{k,n},h}\right|&=\left|\sum_{k=1}^{h}\left(S_{mrh,k}(t)-\sum_{n=0}^{N_{h,k}(t)-1}\mu_{mr,s_{k,n},h}\right)\right|\\
&\qquad\leq\sum_{k=1}^{h}\left|S_{mrh,k}(t)-\sum_{n=0}^{N_{h,k}(t)-1}\mu_{mr,s_{k,n},h}\right|\\
&\qquad\leq\sum_{k=1}^{h}b_{\eta_{mrhk},\rho}(V_{mrh,k}(t)).
\end{align*}
For brevity, write
\[
B_{mrh}^{\mathrm{delay}}(t):=\sum_{k=1}^{h}b_{\eta_{mrhk},\rho}(V_{mrh,k}(t)).
\]
The preceding inequality becomes
\[
-B_{mrh}^{\mathrm{delay}}(t)\leq\sum_{k=1}^{h}S_{mrh,k}(t)-A_{rh}^{\mathrm{delay}}(t)\Gamma_{mrh,t}^{\mathrm{delay}}\leq B_{mrh}^{\mathrm{delay}}(t).
\]
Rearranging,
\[
\sum_{k=1}^{h}S_{mrh,k}(t)-B_{mrh}^{\mathrm{delay}}(t)\leq A_{rh}^{\mathrm{delay}}(t)\Gamma_{mrh,t}^{\mathrm{delay}}\leq\sum_{k=1}^{h}S_{mrh,k}(t)+B_{mrh}^{\mathrm{delay}}(t).
\]
Since $A_{rh}^{\mathrm{delay}}(t)>0$, dividing by $A_{rh}^{\mathrm{delay}}(t)$ gives
\[
L_{mrh,t}^{\mathrm{delay}}\leq\Gamma_{mrh,t}^{\mathrm{delay}}\leq U_{mrh,t}^{\mathrm{delay}},
\]
where
\begin{eqnarray*}
L_{mrh,t}^{\mathrm{delay}}&:=&\frac{\sum_{k=1}^{h}S_{mrh,k}(t)-\sum_{k=1}^{h}b_{\eta_{mrhk},\rho}(V_{mrh,k}(t))}{A_{rh}^{\mathrm{delay}}(t)},\\
U_{mrh,t}^{\mathrm{delay}}&:=&\frac{\sum_{k=1}^{h}S_{mrh,k}(t)+\sum_{k=1}^{h}b_{\eta_{mrhk},\rho}(V_{mrh,k}(t))}{A_{rh}^{\mathrm{delay}}(t)}.
\end{eqnarray*}
If $A_{rh}^{\mathrm{delay}}(t)=0$, set
\[
\Gamma_{mrh,t}^{\mathrm{delay}}:=0,\qquad L_{mrh,t}^{\mathrm{delay}}:=-\infty,\qquad U_{mrh,t}^{\mathrm{delay}}:=+\infty.
\]
Therefore
\[
\Pp\left(\forall t,m,r,h:\ L_{mrh,t}^{\mathrm{delay}}\leq\Gamma_{mrh,t}^{\mathrm{delay}}\leq U_{mrh,t}^{\mathrm{delay}}\right)\geq1-\delta.
\]

\begin{proof}[Proof of Corollary \ref{cor:delay}]
On the preceding joint event, regard each pair $(r,h)$ as a coverage-constraint index and apply the definitions of Section~\ref{sec:method} with the delayed quantities. If a method is certified, then for every monitored $(r,h)$,
\[
\Gamma_{mrh,t}^{\mathrm{delay}}\leq U_{mrh,t}^{\mathrm{delay}}\leq0,
\]
so it is truly feasible. Conversely, true feasibility gives
\[
L_{mrh,t}^{\mathrm{delay}}\leq\Gamma_{mrh,t}^{\mathrm{delay}}\leq0
\]
for every monitored $(r,h)$, so the method is possibly feasible. Thus, for every $t$,
\[
\widehat\Vset_t^{\cert}\subseteq\Vset_t^\star\subseteq\widehat\Vset_t^{\pos}.
\]

Now let $m^\star\in\M_t^\star$. Since every $j\in\widehat\Vset_t^{\cert}$ is truly feasible,
\[
Q_{m^\star,t}\leq Q_{j,t},\qquad \forall j\in\widehat\Vset_t^{\cert}.
\]
Therefore $Q_{m^\star,t}\leq q_t^{\cert}$, with $q_t^{\cert}=+\infty$ when $\widehat\Vset_t^{\cert}=\varnothing$. Since $m^\star\in\widehat\Vset_t^{\pos}$,
\[
m^\star\in\left\{m\in\widehat\Vset_t^{\pos}:Q_{m,t}\leq q_t^{\cert}\right\}=\widehat\M_t^{\proj}.
\]
Hence $\M_t^\star\subseteq\widehat\M_t^{\proj}$ simultaneously over all calendar times and monitored horizons. The joint event has probability at least $1-\delta$, which proves the corollary.
\end{proof}

\subsection{Outcome-Dependent Efficiency Objectives}\label{app:randomobjective}

The main text uses a predictable set-size objective. Suppose instead that method $m$ incurs an adapted loss $\ell_{mt}$ after $Y_t$ is observed, and the target is
\[
\overline Q_{m,t}
=\frac{1}{W_t}\sum_{s=1}^t
\E[w_s\ell_{ms}\mid\G_{s-1}].
\]
For each ordered pair $(m,j)$, define
\[
D_{mj,t}:=w_t(\ell_{mt}-\ell_{jt}),
\qquad
\Delta_{mj,t}:=\overline Q_{m,t}-\overline Q_{j,t}=\frac{1}{W_t}\sum_{s=1}^t\E[D_{mj,s}\mid\G_{s-1}].
\]
Suppose that simultaneous lower confidence sequences $L^\Delta_{mj,t}$ can be constructed such that
\[
\Pp\!\left(\forall t,m,j:\ L^\Delta_{mj,t}\leq\Delta_{mj,t}\right)\geq1-\delta_Q.
\]
Such confidence sequences can be obtained, for example, when the centered pairwise increments admit predictable bounds or an appropriate conditional sub-$\psi$ process.
Combine this event with the feasibility confidence event using total error $\delta_V+\delta_Q\leq\delta$, and define
\begin{equation}
\widehat\M_t^{\mathrm{loss}}=\left\{m\in\widehat\Vset_t^{\pos}:L^\Delta_{mj,t}\leq0\text{ for every }j\in\widehat\Vset_t^{\cert}\right\}.
\label{eq:randomobjective}
\end{equation}
If $m^\star$ is a true constrained optimizer and $j$ is certified, then $j$ is truly feasible and $\Delta_{m^\star j,t}\leq0$; the lower confidence bound cannot be positive on the joint event. Therefore $m^\star\in\widehat\M_t^{\mathrm{loss}}$ for all $t$. Pairwise differences exploit the fact that all methods are evaluated on the same outcomes, but \eqref{eq:randomobjective} is generally an outer set rather than the exact projection of a coherent joint objective region.

\end{document}